\documentclass{article}

\usepackage[final,nonatbib]{nips_2016_modified}

\usepackage[utf8]{inputenc} % allow utf-8 input
\usepackage[T1]{fontenc}    % use 8-bit T1 fonts
\usepackage{hyperref}       % hyperlinks
\usepackage{url}            % simple URL typesetting
\usepackage{nicefrac}       % compact symbols for 1/2, etc.
\usepackage{microtype}      % microtypography
\usepackage{amsmath,listings,xspace,color,amssymb}
\usepackage{amsfonts}
\usepackage{dsfont}
\usepackage{enumitem}
\usepackage{mathtools}
\usepackage{bm}

\usepackage{inconsolata}

\usepackage{etoolbox}
\usepackage{breqn}

\usepackage[export]{adjustbox}

\usepackage{algorithm}
\usepackage[noend]{algpseudocode}

\usepackage{amsthm}

\usepackage{subcaption}
\BeforeBeginEnvironment{dmath}{\begin{nolinenumbers}}
\AfterEndEnvironment{dmath}{\end{nolinenumbers}}
\BeforeBeginEnvironment{dmath*}{\begin{nolinenumbers}}
\AfterEndEnvironment{dmath*}{\end{nolinenumbers}}

\usepackage{booktabs}
\usepackage{multirow}
\usepackage{array}

\usepackage[percent]{overpic}

\usepackage{wrapfig}

\usepackage{bm}

\newcommand{\vecstyle}[1]{\mathbf{#1}}

\newcommand{\indicator}[1]{\ensuremath{\mathds{1}\{#1\}}}

\newcommand{\normdist}{\ensuremath{\mathcal{N}}}

\newcommand{\KLD}{\ensuremath{D_{\text{KL}}}}

\newcommand{\expect}{\mathbb{E}}

\newcommand{\Reals}{\ensuremath{\mathds{R}}}

\newcommand{\ic}[1]{\lstinline[basicstyle=\fontsize{8pt}{8.25pt}\selectfont\ttfamily]{#1}}

\newcommand{\hilite}[1]{\textcolor{SeaGreen}{\mbox{#1}}}

\input{js.def}
\usepackage[dvipsnames]{xcolor}

\usepackage{tikz}
\usetikzlibrary{bayesnet}

\title{Deep Amortized Inference for Probabilistic Programs}

\author{
  Daniel Ritchie\\
  Stanford University
  \And
  Paul Horsfall\\
  Stanford University
  \And
  Noah D. Goodman\\
  Stanford University
}

\begin{document}

%% Notation that's specific to this paper
%% For some reason, this has to be here (i.e. after \begin{document}), otherwise
%%    the arXiv latex build process barfs (fails to recognize commands).
%auto-ignore

\newcommand{\latentVar}{\ensuremath{x}}
\newcommand{\observedVar}{\ensuremath{y}}
\newcommand{\latentVars}{\ensuremath{\vecstyle{\latentVar}}}
\newcommand{\observedVars}{\ensuremath{\vecstyle{\observedVar}}}

% -------------------------------------------------------------------

\newcommand{\trueJoint}{\ensuremath{p(\latentVars, \observedVars)}}
\newcommand{\trueJointPhi}{\ensuremath{p(\latentVars, \observedVars ; \phi)}}
\newcommand{\trueJointTheta}{\ensuremath{p(\latentVars, \observedVars; \theta)}}

\newcommand{\truePosterior}{\ensuremath{p(\latentVars | \observedVars)}}
\newcommand{\truePosteriorPhi}{\ensuremath{p(\latentVars | \observedVars ; \phi)}}
\newcommand{\truePosteriorTheta}{\ensuremath{p(\latentVars | \observedVars ; \theta)}}

\newcommand{\dataMarginal}{\ensuremath{p(\observedVars)}}
\newcommand{\dataMarginalPhi}{\ensuremath{p(\observedVars ; \phi)}}
\newcommand{\dataMarginalTheta}{\ensuremath{p(\observedVars ; \theta)}}

% -------------------------------------------------------------------

% \newcommand{\guide}{\ensuremath{g}}
\newcommand{\guide}{\ensuremath{q}}

\newcommand{\guidePosterior}{\ensuremath{\guide(\latentVars | \observedVars ; \phi)}}
\newcommand{\guidePostNoPhi}{\ensuremath{\guide(\latentVars | \observedVars)}}

\newcommand{\elbo}{\ensuremath{\mathcal{L}(\observedVars)}}
\newcommand{\elboPhi}{\ensuremath{\mathcal{L}(\observedVars, \phi)}}
\newcommand{\elboPhiTheta}{\ensuremath{\mathcal{L}(\observedVars, \phi, \theta)}}

\newcommand{\elboDef}{\ensuremath{\expect_{\guide}[ \log \trueJoint - \log \guidePosterior ]}}
\newcommand{\elboDefGenPhi}{\ensuremath{\expect_{\guide}[ \log \trueJointPhi - \log \guidePosterior ]}}
\newcommand{\elboDefGenTheta}{\ensuremath{\expect_{\guide}[ \log \trueJointTheta - \log \guidePosterior ]}}

\newcommand{\gradparams}{\ensuremath{\nabla_\phi}}
\newcommand{\gradparamsTheta}{\ensuremath{\nabla_\theta}}

% -------------------------------------------------------------------

\newcommand{\reparamVar}{\ensuremath{\epsilon}}
\newcommand{\reparamVars}{\ensuremath{\bm{\epsilon}}}
\newcommand{\reparamDist}{\ensuremath{r}}
\newcommand{\reparamXform}{\ensuremath{g}}
\newcommand{\xformedVar}{\ensuremath{\reparamXform(\reparamVar)}}
\newcommand{\xformedVarPhi}{\ensuremath{\reparamXform(\reparamVar ; \phi)}}
\newcommand{\xformedVars}{\ensuremath{\reparamXform(\reparamVars)}}
\newcommand{\xformedVarsPhi}{\ensuremath{\reparamXform(\reparamVars ; \phi)}}

\maketitle

\begin{abstract}
%auto-ignore

Probabilistic programming languages (PPLs) are a powerful modeling tool, able to represent any computable probability distribution.
Unfortunately, probabilistic program inference is often intractable, and existing PPLs mostly rely on expensive, approximate sampling-based methods.
To alleviate this problem, one could try to learn from past inferences, so that future inferences run faster.
This strategy is known as \emph{amortized inference}; it has recently been applied to Bayesian networks~\cite{StochasticInverses,NeuralStochasticInverses} and deep generative models~\cite{NVIL,AEVB,DLGM}.
This paper proposes a system for amortized inference in PPLs.
In our system, amortization comes in the form of a parameterized \emph{guide program}.
Guide programs have similar structure to the original program, but can have richer data flow, including neural network components.
These networks can be optimized so that the guide approximately samples from the posterior distribution defined by the original program.
We present a flexible interface for defining guide programs and a stochastic gradient-based scheme for optimizing guide parameters, as well as some preliminary results on automatically deriving guide programs.
We explore in detail the common machine learning pattern in which a `local' model is specified by `global' random values and used to generate independent observed data points; this gives rise to amortized local inference supporting global model learning.
%
% \remark{Brief summary of some experimental evaluation.}

\end{abstract}

%% Include files for each section of content
%auto-ignore
\section{Introduction}
\label{sec:introduction}

%% What's our domain?

Probabilistic models provide a framework for describing abstract prior knowledge and using it to reason under uncertainty.
Probabilistic programs are a powerful tool for probabilistic modeling. A probabilistic programming language (PPL) is a deterministic programming language augmented with random sampling and Bayesian conditioning operators. 
Performing inference on these programs then involves reasoning about the space of executions which satisfy some constraints, such as observed values. 
A universal PPL, one built on a Turing-complete language, can represent any computable probability distribution, including open-world models, Bayesian non-parameterics, and stochastic recursion~\cite{Church,Venture,Anglican}.

%% What's the problem?

If we consider a probabilistic program to define a distribution $p(\latentVars, \observedVars)$, where $\latentVars$ are (latent) intermediate variable and $\observedVars$ are (observed) output data, then sampling from this distribution is easy: just run the program forward. However, computing the posterior distribution $p(\latentVars | \observedVars)$ is hard, involving an intractable integral. Typically, PPLs provide means to approximate the posterior using Monte Carlo methods (e.g. MCMC, SMC), dynamic programming, or analytic computation.

%% What's the big insight?

These inference methods are expensive because they (approximately) solve an intractable integral from scratch on every separate invocation.
But many inference problems have shared structure: it is reasonable to expect that computing $p(\latentVars | \observedVars_1)$ should give us some information about how to compute $p(\latentVars | \observedVars_2)$.
In fact, there is reason to believe that this is how people are able to perform certain inferences, such as visual perception, so quickly---we have perceived the world many times before, and can leverage that accumulated knowledge when presented with a new perception task~\cite{AmortizedInference}.
This idea of using the results of previous inferences, or precomputation in general, to make later inferences more efficient is called \emph{amortized inference}~\cite{AmortizedInference,StochasticInverses}.

\emph{Learning} a generative model from many data points is a particularly important task that leads to many related inferences.
One wishes to update global beliefs about the true generative model from individual data points (or batches of data points).
While many algorithms are possible for this task, they all require some form of  `parsing' for each data point: doing posterior inference in the current generative model to guess values of local latent variable given each observation.
Because this local parsing inference is needed many many times, it is a good candidate for amortization.
It is plausible that learning to do local inference via amortization would support faster and better global learning, which gives more useful local inferences, leading to a virtuous cycle.

%% What's our approach?

This paper proposes a system for amortized inference in PPLs, and applies it to model learning. Instead of computing $p(\latentVars | \observedVars)$ from scratch for each $\observedVars$, our system instead constructs a program $\guide(\latentVars | \observedVars)$ which takes $\observedVars$ as input and, when run forward, produces samples distributed approximately according to the true posterior $p(\latentVars | \observedVars)$.
We call $\guide$ a \emph{guide program}, following terminology introduced in previous work~\cite{GuidePrograms}.
The system can spend time up-front constructing a good approximation $\guide$ so that at inference time, sampling from $\guide$ is both fast and accurate.

There is a huge space of possible programs $\guide$ one might consider for this task. Rather than posing the search for $\guide$ as a general program induction problem (as was done in previous work~\cite{GuidePrograms}), we restrict $\guide$ to have the same control flow as the original program $p$, but a different data flow.
That is, $\guide$ samples the same random choices as $p$ and in the same order, but the data flowing into those choices comes from a different computation.
In our system, we represent this computation using neural networks.
This design choice reduces the search for $\guide$ to the much simpler continuous problem of optimizing the weights for these networks, which can be done using stochastic gradient descent.
% The gradients are often high-variance, esp. for programs with discrete choices. To combat this, our system implements several general-purpose variance-reduction strategies.

Our system's interface for specifying guide programs is flexible enough to subsume several popular recent approaches to variational inference, including those that perform both inference and model learning. To facilitate this common pattern we introduce the \ic{mapData} construct which represents the boundary between global ``model'' variables and variables local to the data points. Our system leverages the independence between data points implied by \ic{mapData} to enable mini-batches of data and variance reduction of gradient estimates.
We evaluate our proof-of-concept system on a variety of Bayesian networks, topic models, and deep generative models.

%When guides are not explicitly provided the system defaults to automatically derived mean-field guides. 
%We discuss possible extensions that might extend this support to more complex guides that lead to amortized inferece.

Our system has been implemented as an extension to the WebPPL probabilistic programming language~\cite{WebPPL}. Its source code can be found in the WebPPL repository, with additional helper code at \url{https://github.com/probmods/webppl-daipp}.

%auto-ignore
\section{Background}
\label{sec:background}

\subsection{Probabilistic Programming Basics}
\label{sec:pplbasics}

For our purposes, a probabilistic program defines a generative model $\trueJoint$ of latent variables $\latentVars$ and data $\observedVars$. The model factors as:
%%%
\begin{equation}
\trueJoint = p(\observedVars | \latentVars) \prod_i p(\latentVar_i | \latentVars_{<i})
\label{eq:probProgDef}
\end{equation}
%%%
The prior probability distribution $p(\latentVars)$ decomposes as a product of conditionals $p_i(\latentVar_i | \latentVars_{<i})$, one for each random choice $\latentVar_i$ in the program. The use of $\latentVars_{<i}$ indicates that a random choice can potentially depend on any or all previous choices made by the program.
$p(\observedVars | \latentVars)$ is the likelihood of the data and need not be a proper probability distribution (i.e. unnormalized factors are acceptable).
Note that $\latentVars$ can vary in length across executions: a probabilistic program can sample a variable number of random variables.

Our system is implemented in the probabilistic programming language WebPPL, which we use for examples throughout this paper~\cite{WebPPL}.
WebPPL is a PPL embedded in Javascript; that is, it adds sampling, conditioning, and inference operators to a purely-functional subset of Javascript.
The following example program illustrates its basic features:
\begin{lstlisting}
var model = function() {
   var x = sample(Bernoulli({p: 0.75}));
   var mu = x ? 2 : 0;
   observe(Gaussian({mu: mu, sigma: 1}), 0.5);
   return x;
};

Infer({method: 'MCMC'}, model);
\end{lstlisting}
This program uses MCMC to compute an approximate posterior distribution over the return value of the function \ic{model}. \ic{model} is a simple generative model with one latent Bernoulli variable (\ic{x}) and one observed Gaussian variable, which in this example is observed to have the value \ic{0.5}. The mean of the observed Gaussian variable (\ic{mu}) is dependent on the value of \ic{x}. Since \ic{model} returns \ic{x}, the result of this program is the posterior marginal distribution over the variable \ic{x}.
In the rest of this paper, we will build on this language, adding guide programs, amortized inference, and model-learning constructs.

\subsection{Inference as Optimization: Variational Inference}
\label{sec:background:variational}

Instead of approximating the posterior $\truePosterior$ with a collection of samples, one could instead try to approximate it via a parameterized distribution $\guide_{\observedVars}(\latentVars ; \phi)$ which is itself easy to sample from.
This is the premise behind variational inference~\cite{VariationalInference}.
The goal is to find parameters $\phi$ such that $\guide_{\observedVars}(\latentVars ; \phi)$ is as close as possible to $\truePosterior$, where closeness is typically measured via KL-divergence.

To use variational inference, one must first choose a parameterized family of distributions $\guide$; one common choice is the \emph{mean-field family}:
%%%
\begin{equation*}
\guide^{\textbf{MF}}_{\observedVars}(\latentVars ; \phi) = \prod_i \guide(\latentVar_i ; \phi_i)
\end{equation*}
%%%
This is a fully-factored distribution: it approximates the true posterior as an independent product of parameterized marginals, one for each random variable.
Several existing general-purpose variational inference systems use this scheme~\cite{AVIPP,BBVI}.
This is easy to work with, but it does not capture any of the dependencies between variables that may occur in the true posterior.
This limitation is often acceptable because $\guide_{\observedVars}$ is defined relative to a particular observation set $\observedVars$, and thus the parameters are re-optimized for each new $\observedVars$.
Thus, while this scheme provides an alternative to Monte Carlo methods (e.g. MCMC) that can be faster and more reliable, it still solves each inference problem from scratch.

\subsection{Amortized (Variational) Inference}

As mentioned in Section~\ref{sec:introduction}, \emph{amortized inference} refers to the use of previous inference solutions (or other pre-computation) to solve subsequent inference problems faster.
There exists experimental evidence that people leverage experience from prior inference tasks when asked to solve related ones~\cite{AmortizedInference}.
This idea has inspired research into developing amortized inference systems for Bayesian networks~\cite{StochasticInverses,NeuralStochasticInverses}. These systems model $p(\latentVars | \observedVars)$ by inverting the network topology and attempting to learn the local conditional distributions of this inverted graphical model.

Amortized inference can also be achieved through variational inference.
Instead of defining a parametric family $\guide_{\observedVars}(\latentVars ; \phi)$ which is specific to a given $\observedVars$, we can instead define a general family $\guidePosterior$ which is conditional on $\observedVars$; that is, it takes $\observedVars$ as input.
In this setting, the mean field family no longer applies, as the factors of $\guide$ must now be functions of $\observedVars$.
However, we can extend the mean field family to handle input data by using neural networks (or other `side computations'):
%%%
\begin{equation*}
\guidePosterior = \prod_i \guide(\latentVar_i ; \text{NN}_i(\observedVars ; \phi))
\end{equation*}
%%%
Here, the parameters of each local conditional in the guide $\guide$ are computed via a neural network function $\text{NN}_i$ of $\observedVars$.
This variational family supports amortized inference: one can invest time up-front optimizing the neural network weights such that $\guidePosterior$ is close to $\truePosterior$. When given a never-before-seen $\observedVars$, the guide program forwards $\observedVars$ through the trained networks for fast inference.
Several recent approaches to `neural variational inference' use some instantiation of this design pattern~\cite{NVIL,DLGM,AEVB}.

Such neural guide families are easy to express in our extensions of WebPPL. Our system also allows generalizations of this pattern, such as providing neural nets with previously-made random choices as additional input:
%%%
\begin{equation*}
\guidePosterior = \prod_i \guide(x_i ; \text{NN}_i(\observedVars, \latentVars_{<i} ; \phi))
\end{equation*}
%%%
Here, $\latentVars_{<i}$ are the random choices made before choice $i$ is sampled. Such guide families have the potential to capture posterior dependencies between latent variables
% ~\remark{Which we demonstrate later in the paper}.

\subsection{Variational Model Learning}

The amortized variational inference framework presented above can also be used to learn the parameters of the generative model $p$. If $p$ is also parameterized, i.e. $\truePosteriorTheta$, then its parameters $\theta$ can be optimized along with the parameters $\phi$ of the guide program~\cite{NVIL,DLGM,AEVB}.

Our system supports learning generative model parameters in addition to guide parameters.
In the PPL setting it is natural to think of this as a particular model pattern in which there are global parameters or random choices that affect a local `observation model', which in turn is assumed to generate each data point independently; we call this the mapData model pattern.
We will show below how it is easy to use this pattern to do (regularized) maximum-likelihood learning, variational Bayes, or even different methods within the same model.

%auto-ignore
\section{Specifying Guide Programs}
\label{sec:guideSpec}

In this section, we describe language extensions to WebPPL that allow for the specification of guide programs.
We focus for now on manually-specified guide programs. In Section~\ref{sec:autoGuide}, we build on this interface to automatically derive guide programs.

\subsection{Sampling with Guide Distributions}

Previously, we mentioned that our system restricts guide programs $\guide$ to have the same control flow as the original program $p$, meaning that the guide program samples the same variables in the same order.
Our implementation enforces this restriction by defining the guide program inline with the regular program.
At each \ic{sample} statement, in addition to the distribution that the program $p$ samples from, one can also specify what distribution the guide program $\guide$ should sample from. For example, using the simple program from Section~\ref{sec:pplbasics}:
%%%
\begin{lstlisting}
var model = function() {
   var x = sample(Bernoulli({p: 0.75}), {
      guide: Bernoulli({p: 0.475})
   });
   var mu = x ? 2 : 0;
   observe(Gaussian({mu: mu, sigma: 1}), 0.5);
   return x;
};
\end{lstlisting}
%%%
In this example, the guide samples from a Bernoulli with a different success probability \ic{p}. This particular value happpens to give the true posterior for this program, since this toy problem is simple enough to solve in closed form.
We note that the guide distribution need not be of the same family as the prior distribution; we will see later how this property can be useful.

\subsection{Declaring Optimizable Parameters}

In real problems, we will not know the optimal guides \emph{a prior} and will instead want to learn guides by specifying guide distributions with tunable parameters:
%%%
\begin{lstlisting}
var x = sample(Gaussian({mu: 0, sigma: 1}), {
   guide: Gaussian({mu: param({name: 'guideMu'}), sigma: softplus(param({name: 'guideSigma'}))})
});
\end{lstlisting}
%%%
Here, \ic{param(\{name: nm\})} declares an optimizable, real-valued parameter named \ic{nm}; as we'll see later, this function can also take a \ic{dims} argument specifying the dimensions of a vector, matrix, or tensor-valued parameter.
Since the standard deviation \ic{sigma} of the Gaussian guide distribution must be positive, we use the \ic{softplus}\footnote{$\text{softplus}(x) = \log(\exp(x) + 1)$} function to map the unbounded value returned by \ic{param} to $\Reals^{+}$; our system includes similar transforms for parameters with other domains (e.g. \ic{sigmoid} for parameters defined over the interval $[0, 1]$).
Parameters must be named so they can be disambiguated by the optimization engine.

Using variational parameters directly as the guide distribution parameters (as done above) results in a mean field approximation for the variable \ic{x}, as mentioned in Section~\ref{sec:background:variational}.
We can also compute the guide parameters via a neural network:
% ~\remark{Should we cut this out and only show the adnn version?}
%%%
\begin{lstlisting}
// Observed value
var y = 0.5;

// Neural net setup
var nIn = 1;
var nHidden = 3;
var nOut = 2;

var model = function() {
   // Neural net params
   var W1 = param({dims: [nHidden, nIn], name: 'W1'});
   var b1 = param({dims: [nHidden, 1], name: 'b1'});
   var W2 = param({dims: [nOut, nHidden], name: 'W2'});
   var b2 = param({dims: [nOut, 1], name: 'b2'});

   // Use neural net to compute guide params
   var nnInput = Vector([y]);
   var nnOutput = linear(sigmoid(linear(nnInput, W1, b1)), W2, b2);

   var x = sample(Gaussian({mu: 0, sigma: 1}), {
      guide: Gaussian({mu: T.get(nnOutput, 0), sigma: softplus(T.get(nnOutput, 1))})
   });
   observe(Gaussian({mu: x, sigma: 0.5}), y);
   return x;
};
\end{lstlisting}
%%%
Explicitly declaring parameters for and defining the structure of large neural networks can become verbose, so we can instead use the adnn\footnote{\url{https://github.com/dritchie/adnn}} neural net library to include neural nets in our programs:
%%%
\begin{lstlisting}
// Observed value
var y = 0.5;

// Neural net setup
var guideNet = nn.mlp(1, [
   {nOut: 3, activation: nn.sigmoid},
   {nOut: 2}
], 'guideNet');

var model = function() {
   // Use neural net to compute guide params
   var nnInput = Vector([y]);
   var nnOutput = nnEval(guideNet, nnInput);

   var x = sample(Gaussian({mu: 0, sigma: 1}), {
      guide: Gaussian({mu: T.get(nnOutput, 0), sigma: softplus(T.get(nnOutput, 1))})
   });
   observe(Gaussian({mu: x, sigma: 0.5}), y);
   return x;
};
\end{lstlisting}
%%%
In this case, the \ic{nn.mlp} constructor has created a \ic{guideNet} object with its own parameters; these parameters are registered with the optimization engine when \ic{nnEval} is called.

\subsection{Iterating over Observed Data}

The previous examples have thus far conditioned on a single observation. But real models condition on multiple observations. Our system expresses this pattern with the \ic{mapData} function:
%%%
\begin{lstlisting}
var obs = loadData('data.json');   // List of observations
var guideNet = nn.mlp(1, [
   {nOut: 3, activation: nn.sigmoid},
   {nOut: 2}
], 'guideNet');
var model = function() {
   var mu_x = 0;
   var sigma_x = 1;
   var sigma_y = 0.5;
   var latents = mapData({data: obs, batchSize: 100}, function(y) {
      var nnInput = Vector([y]);
      var nnOutput = nnEval(guideNet, nnInput);
      var x = sample(Gaussian({mu: mu_x, sigma: sigma_x}), {
         guide: Gaussian({mu: T.get(nnOutput, 0), sigma: softplus(T.get(nnOutput, 1))})
      });
      observe(Gaussian({mu: x, sigma: sigma_y}), y);
      return x;
   });
   return latents;
};
\end{lstlisting}
%%%
\ic{mapData} operates much like \ic{map} in a typical functional programming language, but it has two important features: (1) the optimization engine treats every execution of the mapped function as independent, and thus (2) the optimization engine can operate on stochastic mini-batches of the data, sized according to the \ic{batchSize} option.
Property (2) is clearly important for efficient, scalable optimization; we will see in Section~\ref{sec:optimization} how property (1) can also be directly leveraged to improve optimization.

\subsection{Defining Learnable Models}

Thus far we have focused on defining parameterized guides for inference.
Parameterized guides can also be used to make models learnable.
The following three code blocks show possible replacements for line 7 of the previous example, replacing the hardcoded constant \ic{mu_x = 0} with a learnable version:

\lstdefinestyle{learnableModels}{numbers=none,basicstyle=\fontsize{6.5pt}{6.75pt}\selectfont\ttfamily}

\begin{minipage}{0.33\linewidth}
\begin{lstlisting}[style=learnableModels]
// Maximum likelihood
var mu_x = sample(
   ImproperUniform(), {
      guide: Delta({
         v: param({name: 'mu_x'})
      })
   });
// For convenience, equivalent to:
// modelParam({name: 'mu_x'});
\end{lstlisting}
\end{minipage}
\hspace{-1em}
\begin{minipage}{0.33\linewidth}
\begin{lstlisting}[style=learnableModels]
// L2-regularized
// maximum likelihood
var mu_x = sample(
   Gaussian({mu: 0, sigma: 1}), {
      guide: Delta({
         v: param({name: 'mu_x'})
      })
   });
\end{lstlisting}
\end{minipage}
\hspace{-1em}
\begin{minipage}{0.33\linewidth}
\begin{lstlisting}[style=learnableModels]
// Variational Bayes
var mu_x = sample(
   Gaussian({mu: 0, sigma: 1}), {
      guide: Gaussian({
        mu: param({name: 'mu_x_m'}),
        sigma: softplus(param({name: 'mu_x_s'}))
      })
   });
\end{lstlisting}
\end{minipage}

The code in the left block results in maximum likelihood estimation. By using a Delta distribution as a guide, the inference engine will optimize for the single best parameter value (i.e. the center of the Delta distribution). Maximum likelihood behavior comes from using an improper uniform distribution (i.e. distribution with probability one everywhere) as a prior. This pattern is common enough that our system provides convenient shorthand for it (the \ic{modelParam} function).
%%%
In the middle code of block, we demonstrate L2-regularized maximum likelihood learning by replacing the improper uniform prior with a Gaussian prior. The inference engine will still predict a point estimate for \ic{mu_x}, but the Gaussian prior results in L2 regularization.
%%%
Finally, the right block shows a variational Bayesian model: \ic{mu_x} has a Gaussian prior distribution, and the guide samples \ic{mu_x} from an approximate variational Gaussian posterior with optimizable parameters.
This form learns a distribution over generative models, maintaining an estimate of uncertainty about the true model.

Note that we could have alternatively implemented maximum likelihood via a direct parameterization, e.g. \ic{var mu_x = param(\{name: 'mu_x'\})}. However, this style results in $p$ being parameterized in addition to $\guide$. This complicates both the implementation and the theoretical analyses that we show later in this paper.
In contrast, our chosen scheme has only the guide parameterized; learning the model is just part of learning the guide.

\subsection{Examples: Simple Bayesian Networks}

The example code we have built up in this section describes a Bayesian network with one continuous latent variable per continuous observation. Figure~\ref{fig:bn_oneLatent} Top shows the fully assembled code (using maximum likelihood estimation for the generative model parameters), along with a graphical model depiction using the notation of Kigma and Welling~\cite{AEVB}. In this diagram, solid arrows indicate dependencies in the generative model given by the main program, and dashed arrows indicate dependencies in the guide program. $\phi$ is shorthand for all the neural network parameters in the guide program.

Figure~\ref{fig:bn_oneLatent} Bottom shows how to modify this code to instead have one discrete latent variable per observation; this is equivalent to a Gaussian mixture model. In this example, the \ic{simplex} function maps a vector in $\Reals^{n-1}$ to the $n$-dimensional simplex (i.e. a vector whose entries sum to one). This process produces a vector of weights suitable for use as the component probabilities of a discrete random variable.

\lstdefinestyle{smaller}{basicstyle=\fontsize{7pt}{7.5pt}\selectfont\ttfamily}

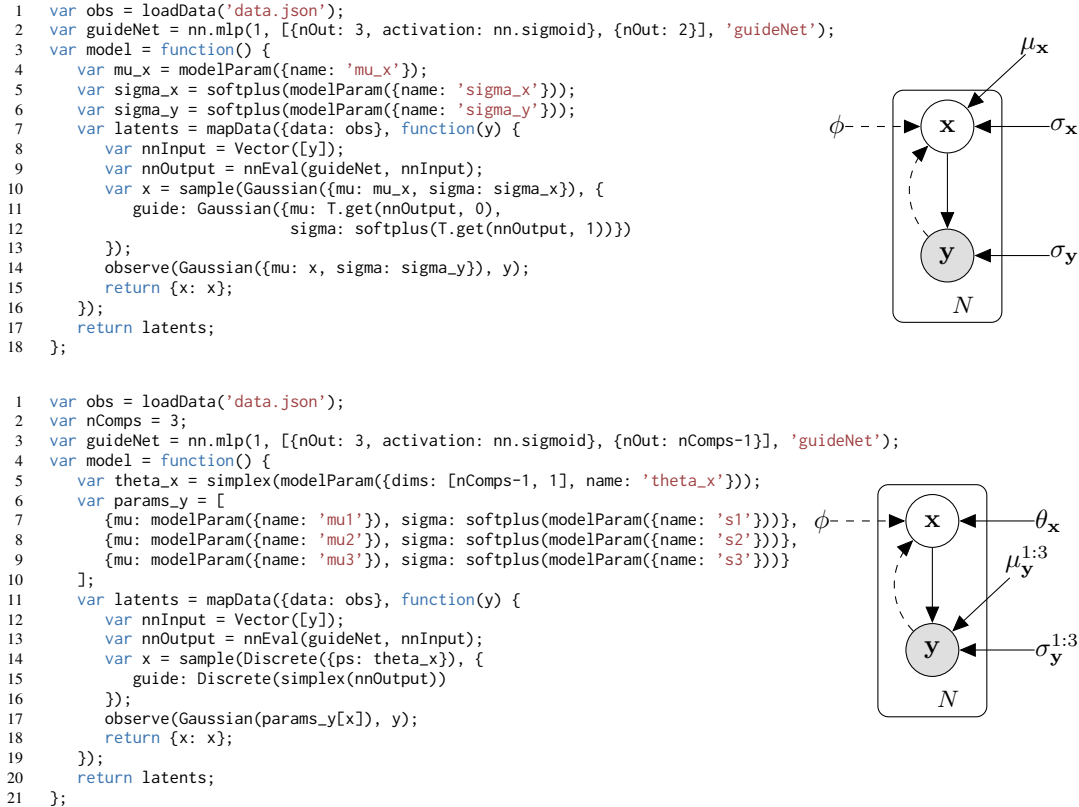
\begin{figure}

% Continuous
\begin{minipage}{\linewidth}
\begin{minipage}{0.66\linewidth}
\begin{lstlisting}[style=smaller]
var obs = loadData('data.json');
var guideNet = nn.mlp(1, [{nOut: 3, activation: nn.sigmoid}, {nOut: 2}], 'guideNet');
var model = function() {
   var mu_x = modelParam({name: 'mu_x'});
   var sigma_x = softplus(modelParam({name: 'sigma_x'}));
   var sigma_y = softplus(modelParam({name: 'sigma_y'}));
   var latents = mapData({data: obs}, function(y) {
      var nnInput = Vector([y]);
      var nnOutput = nnEval(guideNet, nnInput);
      var x = sample(Gaussian({mu: mu_x, sigma: sigma_x}), {
         guide: Gaussian({mu: T.get(nnOutput, 0),
                          sigma: softplus(T.get(nnOutput, 1))})
      });
      observe(Gaussian({mu: x, sigma: sigma_y}), y);
      return {x: x};
   });
   return latents;
};
\end{lstlisting}
\end{minipage}
\begin{minipage}{0.33\linewidth}
\begin{flushright}
\begin{tikzpicture}[scale=1, transform shape]
\node[obs] (y1) {$\mathbf{y}$};
\node[latent, above=of y1] (x1) {$\mathbf{x}$};
\node[const, left=of x1] (phi1) {$\phi$};
\node[const, above right=of x1] (mu_x) {$\mu_{\mathbf{x}}$};
\node[const, right=of x1] (sigma_x) {$\sigma_{\mathbf{x}}$};
\node[const, right=of y1] (sigma_y) {$\sigma_{\mathbf{y}}$};
\edge [dashed] {phi1} {x1};
\edge {mu_x} {x1};
\edge {sigma_x} {x1};
\edge {sigma_y} {y1};
\draw (y1) edge[out=135,in=225,->,dashed] (x1);
\edge {x1} {y1};
\plate [xscale=1.5] {} {(y1)(x1)} {$N$} ;
\end{tikzpicture}
\end{flushright}
\end{minipage}
\end{minipage}

% Discrete
\begin{minipage}{\linewidth}
\begin{minipage}{0.66\linewidth}
\begin{lstlisting}[style=smaller]
var obs = loadData('data.json');
var nComps = 3;
var guideNet = nn.mlp(1, [{nOut: 3, activation: nn.sigmoid}, {nOut: nComps-1}], 'guideNet');
var model = function() {
   var theta_x = simplex(modelParam({dims: [nComps-1, 1], name: 'theta_x'}));
   var params_y = [
      {mu: modelParam({name: 'mu1'}), sigma: softplus(modelParam({name: 's1'}))},
      {mu: modelParam({name: 'mu2'}), sigma: softplus(modelParam({name: 's2'}))},
      {mu: modelParam({name: 'mu3'}), sigma: softplus(modelParam({name: 's3'}))}
   ];
   var latents = mapData({data: obs}, function(y) {
      var nnInput = Vector([y]);
      var nnOutput = nnEval(guideNet, nnInput);
      var x = sample(Discrete({ps: theta_x}), {
         guide: Discrete(simplex(nnOutput))
      });
      observe(Gaussian(params_y[x]), y);
      return {x: x};
   });
   return latents;
};
\end{lstlisting}
\end{minipage}
\begin{minipage}{0.33\linewidth}
\begin{flushright}
\begin{tikzpicture}[scale=1, transform shape]
\node[obs] (y1) {$\mathbf{y}$};
\node[latent, above=of y1] (x1) {$\mathbf{x}$};
\node[const, left=of x1] (phi1) {$\phi$};
\node[const, right=of x1] (theta_x) {$\theta_{\mathbf{x}}$};
\node[const, above right=of y1] (mu_y) {$\mu^{1\colon3}_{\mathbf{y}}$};
\node[const, right=of y1] (sigma_y) {$\sigma^{1\colon3}_{\mathbf{y}}$};
\edge [dashed] {phi1} {x1};
\edge {theta_x} {x1};
\edge {mu_y} {y1};
\edge {sigma_y} {y1};
\draw (y1) edge[out=135,in=225,->,dashed] (x1);
\edge {x1} {y1};
\plate [xscale=1.5] {} {(y1)(x1)} {$N$} ;
\end{tikzpicture}
\end{flushright}
\end{minipage}
\end{minipage}

\caption{WebPPL code and corresponding graphical models for simple Bayesian networks with one latent variable per observation. \emph{Top:} Continuous latent variable. \emph{Bottom:} Discrete latent variable with 3 discrete values (i.e. a Gaussian mixture model with 3 mixture components).}
\label{fig:bn_oneLatent}

\end{figure}

Figure~\ref{fig:bn_twoLatent} Top shows a slightly more complex Bayesian network with two latent continuous variables. Note that the guide program in this example predicts the two latent variables independently given the observation $\vecstyle{y}$. In Figure~\ref{fig:bn_twoLatent} Bottom, we make some small changes to the code (lines 3 and 17, highlighted in green) to instead have the guide program predict the second latent variable $\vecstyle{x}_2$ as a function of the first latent variable $\vecstyle{x}_1$. This small change allows the guide to capture a posterior dependency that was ingored by the first version.
% \remark{We'll show later how this gives better results?}

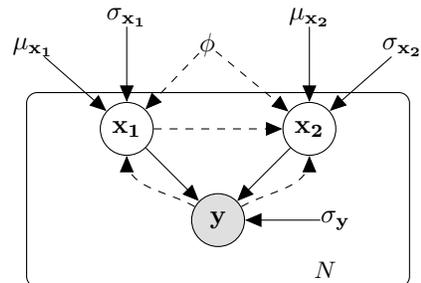
\begin{figure}

% Two continuous; independent
\begin{minipage}{\linewidth}
\begin{minipage}{0.66\linewidth}
\begin{lstlisting}[style=smaller]
var obs = loadData('data.json');
var guideNet1 = nn.mlp(1, [{nOut: 3, activation: nn.sigmoid}, {nOut: 2}], 'guideNet1');
var guideNet2 = nn.mlp(1, [{nOut: 3, activation: nn.sigmoid}, {nOut: 2}], 'guideNet2');
var model = function() {
   var mu_x1 = modelParam('mu_x1');
   var sigma_x1 = softplus(modelParam({name: 'sigma_x1'}));
   var mu_x2 = modelParam('mu_x2');
   var sigma_x2 = softplus(modelParam({name: 'sigma_x2'}));
   var sigma_y = softplus(modelParam({name: 'sigma_y'}));
   var latents = mapData({data: obs}, function(y) {
      var nnInput1 = Vector([y]);
      var nnOutput1 = nnEval(guideNet1, nnInput1);
      var x1 = sample(Gaussian({mu: mu_x1, sigma: sigma_x1}), {
         guide: Gaussian({mu: T.get(nnOutput1, 0),
                          sigma: softplus(T.get(nnOutput1, 1))})
      });
      var nnInput2 = Vector([y]);
      var nnOutput2 = nnEval(guideNet2, nnInput2);
      var x2 = sample(Gaussian({mu: mu_x2, sigma: sigma_x2}), {
         guide: Gaussian({mu: T.get(nnOutput2, 0),
                          sigma: softplus(T.get(nnOutput2, 1))})
      });
      observe(Gaussian({mu: x1 + x2, sigma: sigma_y}), y);
      return {x1: x1, x2: x2};
   });
   return latents;
};
\end{lstlisting}
\end{minipage}
\begin{minipage}{0.33\linewidth}
\begin{flushright}
\begin{tikzpicture}[scale=1, transform shape]
\node[obs] (y1) {$\mathbf{y}$};
\node[latent, above left=of y1] (x1) {$\mathbf{x_1}$};
\node[latent, above right=of y1] (x2) {$\mathbf{x_2}$};
\node[const, above right=of x1] (phi1) {$\phi$};
\node[const, above left=of x1] (mu_x1) {$\mu_{\mathbf{x_1}}$};
\node[const, above=of x1] (sigma_x1) {$\sigma_{\mathbf{x_1}}$};
\node[const, above=of x2] (mu_x2) {$\mu_{\mathbf{x_2}}$};
\node[const, above right=of x2] (sigma_x2) {$\sigma_{\mathbf{x_2}}$};
\node[const, right=of y1] (sigma_y) {$\sigma_{\mathbf{y}}$};
\edge [dashed] {phi1} {x1};
\edge [dashed] {phi1} {x2};
\edge {mu_x1} {x1};
\edge {sigma_x1} {x1};
\edge {mu_x2} {x2};
\edge {sigma_x2} {x2};
\edge {sigma_y} {y1};
\draw (y1) edge[out=150,in=270,->,dashed] (x1);
\draw (y1) edge[out=30,in=270,->,dashed] (x2);
\edge {x1} {y1};
\edge {x2} {y1};
\plate [xscale=1.5] {} {(y1)(x1)(x2)} {$N$} ;
\end{tikzpicture}
\end{flushright}
\end{minipage}
\end{minipage}

% Two continuous; dependent
\begin{minipage}{\linewidth}
\begin{minipage}{0.66\linewidth}
\begin{lstlisting}[style=smaller]
var obs = loadData('data.json');
var guideNet1 = nn.mlp(1, [{nOut: 3, activation: nn.sigmoid}, {nOut: 2}], 'guideNet1');
var guideNet2 = nn.mlp(<@\hilite{2}@>, [{nOut: 3, activation: nn.sigmoid}, {nOut: 2}], 'guideNet2');
var model = function() {
   var mu_x1 = modelParam('mu_x1');
   var sigma_x1 = softplus(modelParam({name: 'sigma_x1'}));
   var mu_x2 = modelParam('mu_x2');
   var sigma_x2 = softplus(modelParam({name: 'sigma_x2'}));
   var sigma_y = softplus(modelParam({name: 'sigma_y'}));
   var latents = mapData({data: obs}, function(y) {
      var nnInput1 = Vector([y]);
      var nnOutput1 = nnEval(guideNet1, nnInput1);
      var x1 = sample(Gaussian({mu: mu_x1, sigma: sigma_x1}), {
         guide: Gaussian({mu: T.get(nnOutput1, 0),
                          sigma: softplus(T.get(nnOutput1, 1))})
      });
      var nnInput2 = <@\hilite{Vector([y, x1]);}@>
      var nnOutput2 = nnEval(guideNet2, nnInput2);
      var x2 = sample(Gaussian({mu: mu_x2, sigma: sigma_x2}), {
         guide: Gaussian({mu: T.get(nnOutput2, 0),
                          sigma: softplus(T.get(nnOutput2, 1))})
      });
      observe(Gaussian({mu: x1 + x2, sigma: sigma_y}), y);
      return {x1: x1, x2: x2};
   });
   return latents;
};
\end{lstlisting}
\end{minipage}
\begin{minipage}{0.33\linewidth}
\begin{flushright}
\begin{tikzpicture}[scale=1, transform shape]
\node[obs] (y1) {$\mathbf{y}$};
\node[latent, above left=of y1] (x1) {$\mathbf{x_1}$};
\node[latent, above right=of y1] (x2) {$\mathbf{x_2}$};
\node[const, above right=of x1] (phi1) {$\phi$};
\node[const, above left=of x1] (mu_x1) {$\mu_{\mathbf{x_1}}$};
\node[const, above=of x1] (sigma_x1) {$\sigma_{\mathbf{x_1}}$};
\node[const, above=of x2] (mu_x2) {$\mu_{\mathbf{x_2}}$};
\node[const, above right=of x2] (sigma_x2) {$\sigma_{\mathbf{x_2}}$};
\node[const, right=of y1] (sigma_y) {$\sigma_{\mathbf{y}}$};
\edge [dashed] {phi1} {x1};
\edge [dashed] {phi1} {x2};
\edge {mu_x1} {x1};
\edge {sigma_x1} {x1};
\edge {mu_x2} {x2};
\edge {sigma_x2} {x2};
\edge {sigma_y} {y1};
\draw (y1) edge[out=150,in=270,->,dashed] (x1);
\draw (y1) edge[out=30,in=270,->,dashed] (x2);
\edge [dashed] {x1} {x2};
\edge {x1} {y1};
\edge {x2} {y1};
\plate [xscale=1.5] {} {(y1)(x1)(x2)} {$N$} ;
\end{tikzpicture}
\end{flushright}
\end{minipage}
\end{minipage}

\caption{WebPPL code and corresponding graphical models for simple Bayesian networks with two latent variables per observation. \emph{Top:} Guide program predicts the two latents independently. \emph{Bottom:} Changing the guide program to treat the second latent variable as conditional on the first (green highlights show changes to the code).}
\label{fig:bn_twoLatent}
\end{figure}

%auto-ignore

\section{Optimizing Parameters}
\label{sec:optimization}

Now that we have seen how to author learnable guide programs, we will describe how to optimize the parameters of those programs. 

\subsection{ELBo: The Variational Objective}

In Section~\ref{sec:background:variational}, we mentioned that the goal of variational inference is to find values of the parameters $\phi$ for our guide program $\guidePosterior$ such that it is as close as possible to the true posterior $\truePosterior$, where closeness is measured via KL-divergence. The KL-divergence between two general distributions is intractable to compute; however, some straightforward algebra produces an objective that is tractable (following the derivation of Wingate and Weber~\cite{AVIPP}):
%%%
\begin{align}
\begin{split}
\KLD(\guidePosterior || \truePosterior)
&= \int_{\latentVars} \guidePosterior \log \frac{\guidePosterior}{\truePosterior}\\
&= \int_{\latentVars} \guidePosterior \log \frac{\guidePosterior}{\trueJoint} + \log \dataMarginal\\
&= \expect_{\guide}[ \log(\guidePosterior -  \trueJoint ) ] + \log \dataMarginal\\
&= \log \dataMarginal - \elboDef\\
&= \log \dataMarginal - \elboPhi \geq 0
\end{split}
\label{eq:elbo}
\end{align}
%%%
where the last inequality follows because KL-divergence is non-negative. This in turn implies that $\elboPhi$ is a lower bound on the log marginal likelihood of the data (i.e. evidence) $\log \dataMarginal$. Accordingly, $\elboPhi$ is sometimes referred to as the `Evidence Lower Bound', or ELBo~\cite{BBVI}. Maximizing the ELBo corresponds to minimizing the KL-divergence.
When $\log p(\observedVars | \latentVars)$ can include unnormalized factors (as in our implementation), this is more properly called the variational free energy.

% When learning a generative model $\trueJointTheta$, the objective is usually to maximize the log marginal likelihood of the training data $\log \dataMarginalTheta$. Rewriting the result from Equation~\ref{eq:elbo}:
% %%%
% \begin{equation*}
% \elboPhiTheta = \log \dataMarginalTheta - \KLD(\guidePosterior || \truePosteriorTheta)
% \end{equation*}
% %%%
% we can see that, again, because KL-divergence is non-negative, maximizing the ELBo also corresponds to maximizing the log marginal likelihood of the data. Thus, whether we are doing amortized infernence only or amortized inference plus model learning, we can use the same optimization objective: the ELBo.
% % \ndg{maybe instead of justifying elbow for model learning and deriving its gradient estimator below, we should just show here that MLE model learning can be cast as param learning with a delta guide and improper prior, and the do the general derivations for q?}
% \remark{remove this once the model learning section has been updated with the improper prior stuff.}

For an alternative derivation of the ELBo using Jensen's inequality, see Mnih and Gregor~\cite{NVIL} and Jordan et al.~\cite[p. 213]{VariationalInference}.

\subsection{ELBo Gradient Estimators}

Maximizing the ELBo requires estimating its gradient with respect to the parameters.
% For model learning, computing the gradient with respect to the model parameters $\theta$ is straightforward:
% %%%
% \begin{align}
% \begin{split}
% \gradparamsTheta \elboPhiTheta
% &= \gradparamsTheta \elboDefGenTheta\\
% &= \expect_\guide [ \gradparamsTheta \log \trueJointTheta ] 
% \end{split}
% \label{eq:modelParamGrad}
% \end{align}
% %%%
% where the expectation with respect to $\guide$ is approximated by samples drawn from $\guide$. This requires that $\trueJointTheta$ be differentiable with respect to $\theta$; many generative models with differentiable likelihoods have this property.
% \remark{remove this once the model learning section has been updated with the improper prior stuff.}
% Estimating the gradient with respect to the guide parameters $\phi$ is more complex.
There are two well-known approaches to performing this estimation:

\paragraph{Likelihood Ratio (LR) Estimator:}
In the general case, the gradient of the ELBo with respect to $\phi$ can be estimated by:
%%%
\begin{align}
\begin{split}
\gradparams \elboPhi
&= \gradparams \elboDef\\
&= \expect_\guide[ \gradparams \log \guidePosterior ( \log \trueJoint - \log \guidePosterior ) ]
\end{split}
\label{eq:lr}
\end{align}
%%%
This estimator goes by the names `likelihood ratio estimator'~\cite{LikelihoodRatioEstimator} and `score function estimator'~\cite{ScoreFunctionEstimator}, and it is also equivalent to the REINFORCE policy gradient algorithm in the reinforcement learning literature~\cite{REINFORCE}. The derivations of this estimator most relevant to our setting can be found in Wingate and Weber~\cite{AVIPP} and Mnih and Gregor~\cite{NVIL}.
Intuitively, each gradient update step using the LR estimator pushes the parameters $\phi$ in the direction $\gradparams \log \guidePosterior$---that is, the direction that will maximize the probability of $\latentVars$ under the guide. But since the goal of optimization is to learn an approximation to the true posterior, this update is weighted based on how probable $\latentVars$ is under the joint distribution $\trueJoint$ (which is proportional to the true posterior). If $\latentVars$ has high probability under the joint but low probability under the current guide, then the $( \log \trueJoint - \log \guidePosterior )$ term produces a large update (i.e. the guide should assign much higher probability to $\latentVars$ than it currently does). If the joint and the guide assign equal probability to $\latentVars$, then the update will have zero magnitude. If the joint assigns \emph{lower} probability to $\latentVars$ than the guide does, the resulting gradient update will move the parameters $\phi$ in the opposite direction of $\gradparams \log \guidePosterior$.

The LR estimator is straightforward to compute, requiring only that $\log \guidePosterior$ be differentiable with respect to $\phi$ (the mean field and neural guide families presented in Section~\ref{sec:background} satisfy this property). However, it is known to exhibit high variance. This problem is amenable to several variance reduction techniques, some of which we will employ later in this section.

\paragraph{Pathwise (PW) Estimator:}

Equation~\ref{eq:lr} suffers from high variance because the gradient can no longer be pushed inside the expectation: the expectation is with respect to $\guide$, and $\guide$ depends on the parameters $\phi$ with respect to which we are differentiating.
However, in certain cases, it is possible to re-write the ELBo such that the expectation distribution does not depend on $\phi$.
This situation occurs whenever the latent variables $\latentVars$ can be expressed as samples from an un-parameterized distribution, followed by a parameterized deterministic transformation:
%%%
\begin{equation*}
\latentVars = \xformedVarsPhi \hspace{2em} \reparamVars \sim \reparamDist(\cdot)
\end{equation*}
%%%
For example, sampling from a Gaussian distribution $\normdist(\mu, \sigma)$ can be expressed as $\mu + \sigma \cdot \reparamVar$, where $\reparamVar \sim \normdist(0, 1)$. Continuous random variables which are parameterized by a location and a scale parameter naturally support this type of transformation, and other types of continuous variables can often be well-approximated by deterministic transformations of unit normal variables~\cite{ADVI}.

Using this `reparameterization trick'~\cite{AEVB} allows the ELBo gradient to be rewritten as:
%%%
\begin{align}
\begin{split}
\gradparams \elboPhi
&= \gradparams \elboDef\\
&= \gradparams \expect_\reparamDist [ \log p(\xformedVarsPhi, \observedVars) - \log \guide(\xformedVarsPhi | \observedVars ; \phi) ]\\
&= \expect_\reparamDist [\gradparams ( \log p(\xformedVarsPhi, \observedVars) - \log \guide(\xformedVarsPhi | \observedVars ; \phi) ) ]
\end{split}
\label{eq:pw}
\end{align}
%%%
This estimator is called the `pathwise derivative estimator'~\cite{PathwiseEstimator}.
It transforms both the guide and target distributions into distributions over independent random `noise' variables $\reparamVars$, followed by complex, parameterized, deterministic transformations. Given a fixed assignment to the noise variables, derivatives can propagate from the final log probabilities back to the input parameters, leading to much more stable gradient estimates than with the LR estimator.

\subsection{Unified Gradient Estimator for Probabilistic Programs}
\label{sec:optimization:unifiedEstimator}

% A general probabilistic program makes many random choices, some of which are amenable to the reparameterization trick (i.e. continuous choices) and others of which are not (i.e. discrete choices).
A general probabilistic program makes many random choices, some of which are amenable to the reparameterization trick and others of which are not.
Discrete random choices are never reparameterizable. Continuous random choices are reparameterizable if they can be expressed as a parameterized, deterministic transformation of a parameter-less random choce. In our implementation, every continuous random choice type either has this property or is well-approximated (and thus can be guided) by a random choice type that does (see Appendix~\ref{sec:appendix_reparam}). Thus, for the rest of the paper, we will equate  continuous with reparameterizable and discrete with non-reparameterizable.
To optimize the ELBo for a probabilistic program $p$ and an associated guide program $\guide$, we seek a single, unified gradient estimator that handles both discrete and continuous choices.
% and supports model learning as well as amortized inference.

% As a first simplifying step, we assume that the generative model and the guide are parameterized by a common set of parameters $\phi$: this reflects the fact that our system internally makes no `generative vs. guide' distinction between parameters.
First, to simplify notation, we drop the dependency on $\phi$ from all derivations that follow. This is equivalent to making the parameters globally available, which is also true of our system (i.e. \ic{param} can be called anywhere).

Second, we assume that all random choices made by the guide program are first drawn from a distribution $\reparamDist$ and then transformed by a deterministic function $\reparamXform$. Under this assumption, the distribution defined by a guide program factors as:
%%%
\begin{equation*}
\guidePostNoPhi = \prod_i \guide(\reparamXform_i(\reparamVar_i) | \reparamXform(\reparamVars_{<i}), \observedVars) \hspace{2em} \reparamVar_i \sim \reparamDist_i(\cdot)
\end{equation*}
%%%
As mentioned in Section~\ref{sec:pplbasics}, the length of $\reparamVars$ can vary across executions. However, since the guide program $\guide$ by construction samples the same random choices in the same order as the target program $p$, $\guide$ factors the same way as $p$ (Equation~\ref{eq:probProgDef}) for any given execution. 
(Also note that the gradient with respect to parameters that affect only variables that are not created on a given execution will be zero, allowing us to simply ignore them.)

The distributions $\reparamDist_i$ and functions $\reparamXform_i$ have different meanings depending on whether the variable $\reparamVar_i$ is continuous or discrete:
%%%
\begin{itemize}
\item{Continuous $\reparamVar_i$: $\reparamDist_i$ is a unit normal or uniform distribution, and $\reparamXform_i$ is a parameterized transform. This is a direct application of the reparameterization trick. In this case, each local transformation $\reparamXform_i$ may also depend on the previous noise variables $\reparamVars_{<i}$, as choices occuring later in the program may be compound transformations of earlier choices.}
\item{Discrete $\reparamVar_i$: $\reparamDist_i$ = $\guide_i$, and $\reparamXform_i$ is the identity function. This allows discrete choices to be represented in the reparameterization trick framework (without actually reparametrizing).}
\end{itemize}
%%%
Given these assumptions, we can derive an estimator for the ELBo gradient:
%%%
\begin{align}
\label{eq:hybridEstimator}
\gradparams \elbo
&= \gradparams \expect_\reparamDist [ \log p(\xformedVars, \observedVars) - \log \guide(\xformedVars | \observedVars) ]\\
&= \expect_\reparamDist [ \gradparams \log \reparamDist(\reparamVars | \observedVars) ( \log p(\xformedVars, \observedVars) - \log \guide(\xformedVars | \observedVars) ) + \gradparams( \log p(\xformedVars, \observedVars) - \log \guide(\xformedVars | \observedVars) )]\nonumber\\
&= \expect_\reparamDist [ \underbrace{\gradparams \log \reparamDist(\reparamVars | \observedVars) W(\reparamVars, \observedVars)}_{\text{LR term}} + \underbrace{\gradparams \log p(\xformedVars, \observedVars) - \gradparams \log \guide(\xformedVars | \observedVars)}_{\text{PW term}} ]\nonumber
\end{align}
%%%
where $W(\reparamVars, \observedVars) = \log p(\xformedVars, \observedVars) - \log \guide(\xformedVars | \observedVars)$; see Appendix~\ref{sec:appendix:estDerivation} for the derivation. 

This estimator includes the LR and PW estimators as special cases. If all random choices are reparameterized (i.e. they are all continuous), then $\reparamDist(\reparamVars | \observedVars)$ does not depend on $\phi$, thus $\gradparams \log \reparamDist(\reparamVars | \observedVars)$ is zero and the LR term drops out, leaving only the PW term. If no random choices are reparameterized (i.e. they are all discrete), then the $\gradparams \log \guide(\xformedVars | \observedVars)$ term drops out, using the identity $\expect_f[\nabla \log f(x)] = 0$ (see Appendix~\ref{sec:appendix:zeroexp}). The $\gradparams \log p(\xformedVars, \observedVars)$ term is also zero, since only $\guide$ and not $p$ is dependent on $\phi$, which leaves only the LR term.

While Equation~\ref{eq:hybridEstimator} is a correct estimator for the ELBo gradient, like the LR estimator, the presence of discrete (i.e. non-reparameterized) random choices can lead to high variance. Thus we modify this estimator through three variance reduction techniques:
%%%
\begin{enumerate}
\item{Replace $W(\reparamVars, \observedVars)$ with a separate $w_i(\reparamVars, \observedVars)$ for each factor $\reparamDist(\reparamVar_i | \reparamVars_{<i}, \observedVars)$ of $\reparamDist(\reparamVars | \observedVars)$, as there exist independencies that can be exploited to reveal zero-expectation terms in $W(\reparamVars, \observedVars)$ for each $i$.}
\item{Subtract a constant `baseline' term $b_i$ from each $w_i(\reparamVars, \observedVars$). This does not change the expectation, but it can reduce its variance, if designed carefully.}
\item{Factor $\gradparams \log \guide(\xformedVars | \observedVars)$ in the PW term and remove factors corresponding to discrete (i.e. non-reparameterized) choices, since, as noted above, they have zero expectation.}
\end{enumerate}
%%%
Further details and correctness proofs for these three steps can be found in Appendix~\ref{sec:appendix_proofs}. Applying them leads to the following estimator, which is the estimator actually used by our system:
%%%
\begin{align}
\begin{split}
% \gradparams \elbo = \expect_\reparamDist \Bigl[ \hspace{0.5em}
% &\sum_{i \in \mathcal{D}} \gradparams \log \reparamDist(\reparamVar_i | \reparamVars_{<i}, \observedVars) ( w_i(\reparamVars, \observedVars) - b_i ) \\
% %
% - &\sum_{i \in \mathcal{C}} \gradparams \log \guide(\reparamXform_i(\reparamVar_i) | \reparamXform_{<i}(\reparamVars_{<i}), \observedVars) \\
% %
% + &\gradparams \log p(\xformedVars, \observedVars) \hspace{0.5em} \Bigr]
\gradparams \elbo
&= \expect_\reparamDist \Bigl[ \hspace{0.5em}
\sum_{i} \gradparams \log \reparamDist(\reparamVar_i | \reparamVars_{<i}, \observedVars) ( w_i(\reparamVars, \observedVars) - b_i )
+ \gradparams \log p(\xformedVars, \observedVars) \hspace{0.5em}
- \sum_{i \in \mathcal{C}} \gradparams \log \guide(\reparamXform_i(\reparamVar_i) | \reparamXform(\reparamVars_{<i}), \observedVars)
\Bigr]\\
&= \expect_\reparamDist \Bigl[ \hspace{0.5em}
\sum_{i \in \mathcal{D}} \gradparams \log \guide(\reparamXform_i(\reparamVar_i) | \reparamXform(\reparamVars_{<i}), \observedVars) ( w_i(\reparamVars, \observedVars) - b_i )
+ \gradparams \log p(\xformedVars, \observedVars) \hspace{0.5em}
- \sum_{i \in \mathcal{C}} \gradparams \log \guide(\reparamXform_i(\reparamVar_i) | \reparamXform(\reparamVars_{<i}), \observedVars)
\Bigr]\\
&= \expect_\reparamDist \Bigl[ \hspace{0.5em}
\gradparams \log p(\xformedVars, \observedVars) \hspace{0.5em}
- \sum_{i} \gradparams \log \guide(\reparamXform_i(\reparamVar_i) | \reparamXform(\reparamVars_{<i}), \observedVars) v_i(\reparamVars, \observedVars)
\Bigr]\\
v_i(\reparamVars, \observedVars) &=
\begin{cases}
- ( w_i(\reparamVars, \observedVars) - b_i ) & \text{if } i \in \mathcal{D}\\
1 & \text{if } i \in \mathcal{C}
\end{cases}
\end{split}
\label{eq:finalEstimator}
\end{align}
%%%
where $\mathcal{C}$ and $\mathcal{D}$ are the sets of indices for all continuous and discrete random choices in the program execution, respectively.
In the second line, we used the fact (mentioned above) that $\gradparams \log \reparamDist(\reparamVar_i | \reparamVars_{<i}, \observedVars) = 0$ for reparameterized (continuous) choices, and that $\reparamDist = \guide$ and $\reparamXform$ is the identity for non-reparameterized (discrete) choices.

Equation~\ref{eq:finalEstimator} is similar to the `surrogate loss function' gradient estimator of Schulman et al.~\cite{StochasticComputationGraphs}, as an execution of a probabilistic program corresponds to one of their stochastic computation graphs. Their analysis is concerned with general stochastic objectives, however, while we focus particularly on the ELBo.

% \ndg{I said elsewhere: Argue that this deals properly with open worlds (ie can ignore rvs / params that weren’t touched on a given sample from q). Only (lazily) create parameter when you first hit it, only update it when RVs that it depends on are hit. Need lemma: gradient for param is zero if none of the RV’s it flows into are hit during a run. Again, with our AD, you don’t have to do any special cases to deal with this, it just falls out naturally.}

% \subsection{Implementation}

% \remark{(Daniel) I feel like we need to bring this section home by making it concrete, describing (in text and/or pseudocode) how we actually calculate the gradient as the program runs. It's easy to describe a simple version of this process (i.e. AD, accumulating into logr, logq, and logp), but the \emph{real} version (using variance reduction + mapData) is more complex: it requires the `dependency' graph, the mapData node stack, and mapData mini-batch multipliers. So I'm conflicted: I want to give the reader something, but a simple version may be misleadingly simple while the full truth may be too complicated to describe. Thoughts? Maybe this section belongs in the Appendix, if it belongs anywhere?}

\subsection{Optimization Interface}

In Section~\ref{sec:pplbasics}, we showed how WebPPL programs use the \ic{Infer} function to perform non-amortized inference on a \ic{model} function. To optimize parameters for amortized inference, WebPPL provides an \ic{Optimize} function with a similar interface:
%%%
\begin{lstlisting}
var model = function() {
   // Use sample, guide, mapData, etc.
};

var params = Optimize(model, {
   steps: 100,
   optMethod: 'adam'
});
\end{lstlisting}
%%%
The code above performs 100 gradient update steps on \ic{model} using the Adam stochastic optimization method~\cite{Adam}.
The return value \ic{params} of this function is a map from parameter names to optimized parameter values.

%auto-ignore
\section{Using Learned Parameters}
\label{sec:usingLearnedGuides}

Given a set of learned parameters, our system can predict latent variables, generate synthetic data, or further refine parameters for a new dataset or a new model.

\subsection{Predicting Latent Variables}

A learned guide can be used to make inferences about new, never-before-seen observations. As an example, we'll use the Gaussian mixture model program in Figure~\ref{fig:bn_oneLatent} Bottom and show how to predict cluster assignments for new observations. Note that the observations used by the program are assigned to the \ic{obs} variable, which is then passed to \ic{mapData}. WebPPL is a purely functional language, so it does not support assigning a new dataset to \ic{obs}. However, it does provide a special \ic{globalStore} object whose fields can be re-assigned. With this in mind, we modify the Gaussian mixture model program as follows\footnote{An alternative to using \ic{globalStore} for mutation would be to re-create the \ic{model} function, closing over the test data instead of the training data.}:
%%%
\begin{lstlisting}
globalStore.data = loadData('data.json');
// Set up guide neural net
var model = function() {
   // Set up generative parameters
   var latents = mapData({data: globalStore.data}, function(y) {
      // Guided sample latents, observe data points
   });
   return latents;
};
\end{lstlisting}
%%%
Now we can easily swap datasets using \ic{globalStore.data}. Given a set of learned parameters \ic{params} for this program, we can obtain a sample prediction for the latent variables for a new dataset:
%%%
\begin{lstlisting}
globalStore.data = loadData('data_test.json');	// Load new test data set

// Forward sample from the guide
sample(Infer({method: 'forward', guide: true, params: params}, model));

// Use the guide as a Sequential Monte Carlo importance sampler
sample(Infer({method: 'SMC', particles: 100, params: params}));
\end{lstlisting}
%%%
We can make predictions either by running the guide program forward, or if the true posterior is very complex and the learned guide only partially approximates it, we can use the guide program as an importance sampler within Sequential Monte Carlo.

\subsection{Generating Synthetic Data}

Forward sampling from the guide can also be used to generate synthetic data from the learned distribution. If we make a slightly modified version of the Gaussian mixture model (call it \ic{modelGen}) that samples data instead of observing it, we can used forward sampling with the optimized parameters \ic{params} to synthesize new data points:
%%%
\begin{lstlisting}
var modelGen = function() {
   var theta_x = simplex(modelParam({dims: [nComps-1, 1], name: 'theta_x'}));
   var params_y = [
      {mu: modelParam({name: 'mu1'}), sigma: softplus(modelParam({name: 's1'}))},
      {mu: modelParam({name: 'mu2'}), sigma: softplus(modelParam({name: 's2'}))},
      {mu: modelParam({name: 'mu3'}), sigma: softplus(modelParam({name: 's3'}))}
   ];
   var x = sample(Discrete({ps: theta_x}));
   return sample(Gaussian(params_y[x]));
};

sample(Infer({method: 'forward', guide: true, params: params}, modelGen));
\end{lstlisting}

\subsection{Further Optimization}
\label{sec:furtherOptim}

A set of learned parameters \ic{params} can also be passed back into \ic{Optimize} for further optimization:
%%%
\begin{lstlisting}
var newParams = Optimize(model, {
   steps: 100,
   optMethod: 'adam',
   params: params
});
\end{lstlisting}
%%%
This can be useful for e.g. fine-tuning existing parameters for a new dataset. Indeed, \ic{model} does not even need to be the same program that was originally used to learn \ic{params}; it just needs to declare some parameters (via \ic{param}) with the same names as parameters in \ic{params}. This can be useful for, for example, making a modification to an existing model without having to re-train its guide program from scratch, or for bootstrap training from a simpler model to a more complex one.

%auto-ignore
\section{Experiments}
\label{sec:results}

Having detailed how to specify and optimize guide programs in our system, in this section, we experimentally evaluate how well programs written in our system can learn generative models and approximate posterior samplers. Unless stated otherwise, we use the following settings for the experiments in this section:
%%%
\begin{itemize}
\item{The Adam optimization method~\cite{Adam} with $\alpha = 0.1$, $\beta_1 = 0.9$, and $\beta_2 = 0.999$.}
\item{One sample from $\guide$ per optimization step to estimate the expectation in Equation~\ref{eq:finalEstimator}.}
\end{itemize}

\subsection{Gaussian Mixture Model}
\label{sec:results_gmm}

We first consider the simple Gaussian mixture model program from Figure~\ref{fig:bn_oneLatent} Bottom. This program samples discrete random choices, so its gradient estimator will include an LR term. Alternatively, we could re-write the program slightly to explicitly marginalize out the discrete random choices; see Appendix~\ref{sec:appendix_code:gmmSumOut}.
Marginalizing out of these choices leads to a tighter bound on the marginal log likelihood, so we would expect this version of the program to achieve a higher ELBo.
As an extra benefit, the gradient estimator for this program then reduces to the PW estimator, which will have lower variance. These benefits come at the cost of amortized inference, however, as this version of the program does not have a guide which can predict the latent cluster assignment given an observed point. We also consider a non-amortized, mean field version of the program for comparison.

Figure~\ref{fig:gmmResults} illustrates the performance of these programs after training for 200 steps on a synthetic datset of 100 points. On the left, we show how the ELBo changes during optimizaton.
As expected, ELBo progress asymptotes at a higher value for the marginalized model.
% Even on a simple model such as this one, the ability to use the pure PW estimator leads to faster convergence and a better final objective score.
On the right, we show the estimated negative log likelihood of a separate synthetic test set under each program after optimization. Here, we also include the true model (i.e. the model used to synthesize the test/training data) for comparsion.
As suggested by its optimization performance, the model with discrete choices marginalized out performs best.
Note that the amortized guide program slightly out-performs the mean field guide program, indicating that the generalization provided by amortization has benefits for training generative models, in addition to enabling fast predictions of latent variables for previously-unseen observations.

\begin{figure}[!ht]
\begin{minipage}{0.5\linewidth}
\centering
\includegraphics[width=\linewidth]{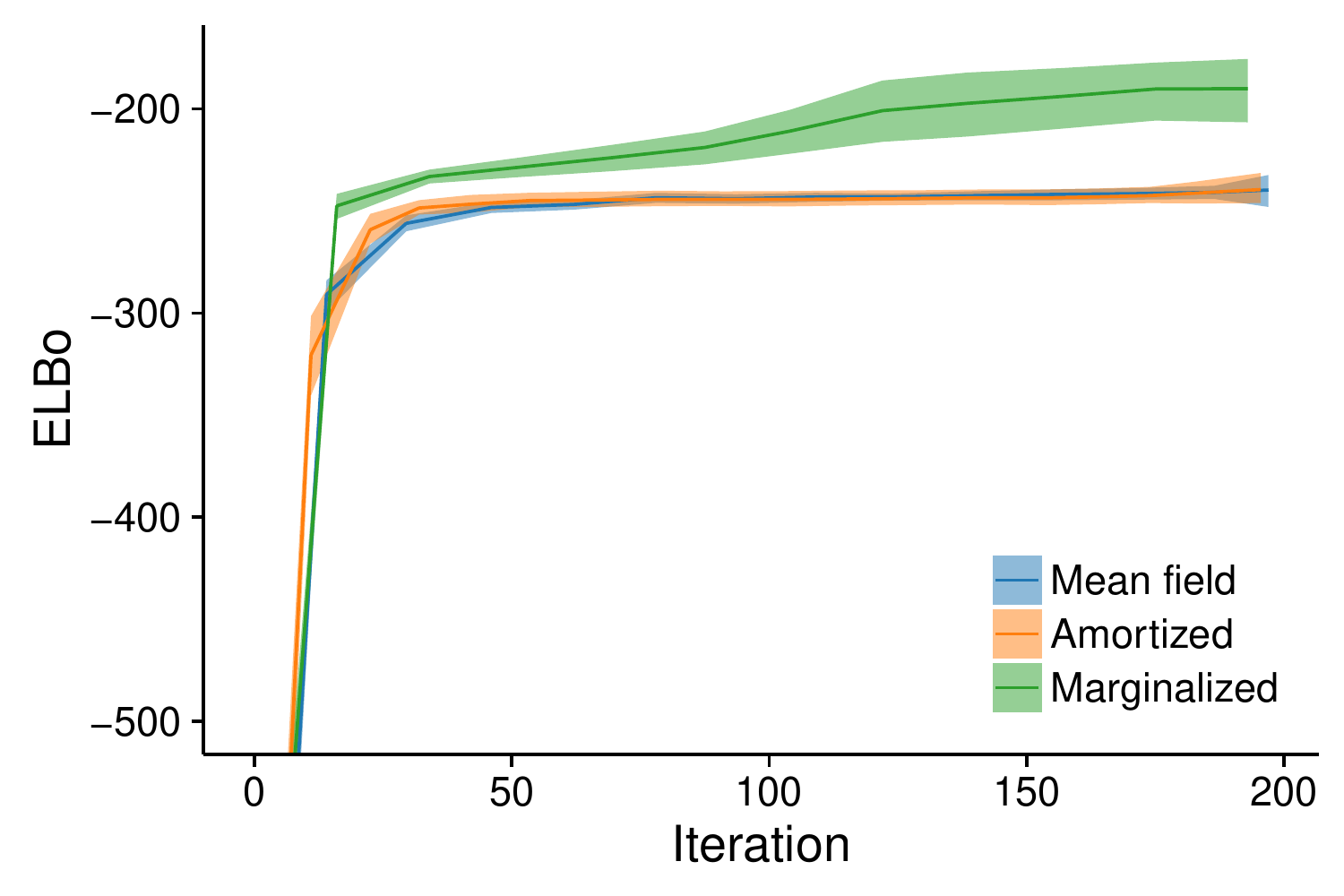}
\end{minipage}
\begin{minipage}{0.5\linewidth}
\centering
\includegraphics[width=\linewidth]{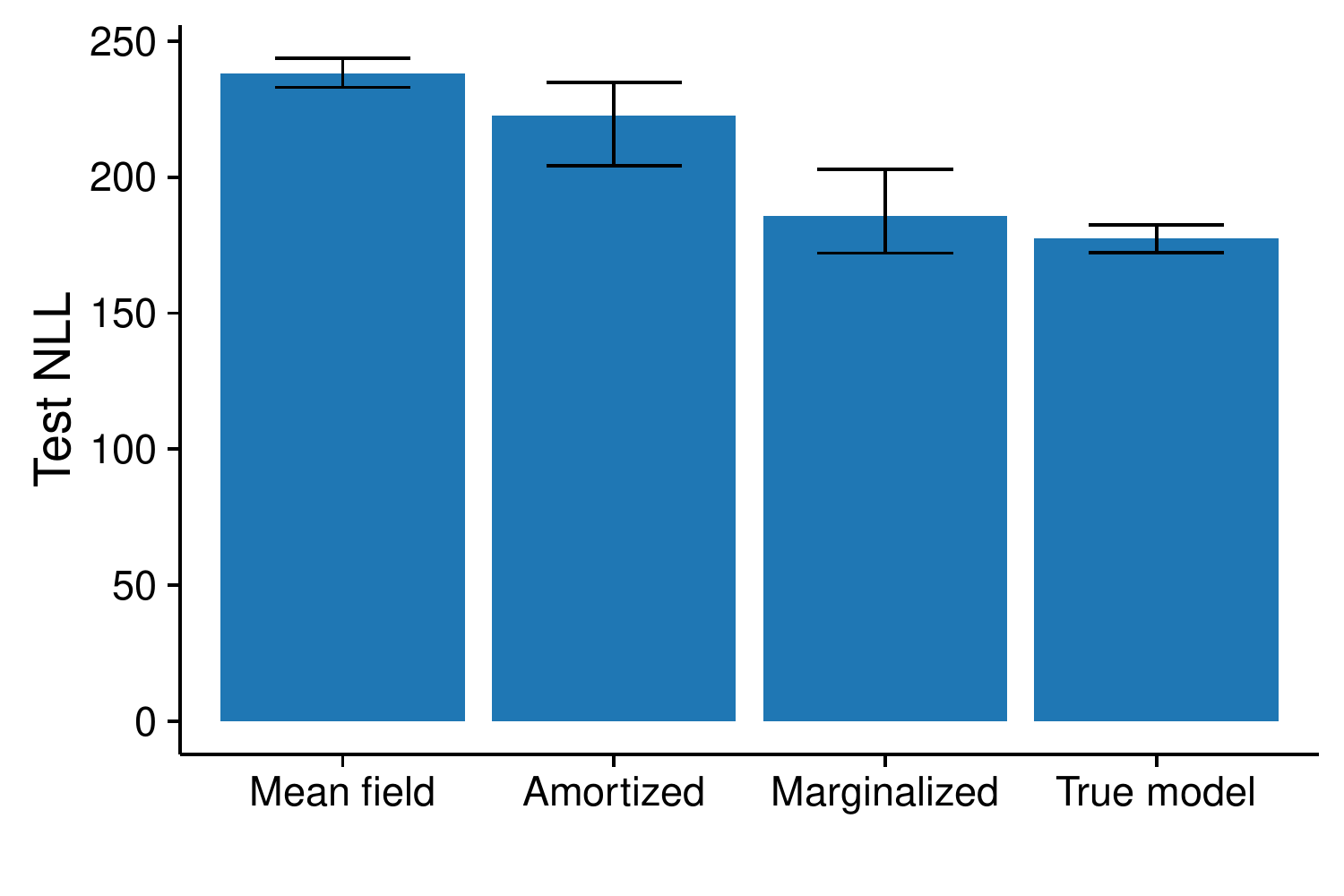}
\end{minipage}
\caption{Performance of simple Gaussian mixture model program. \emph{(Left)} ELBo optimization progress during training. The optimization objective for the marginalized model is a tighter bound on the marginal log likelihood and thus has a higher asymptote. \emph{(Right)} Negative log-likelihood of a held-out test set.}
\label{fig:gmmResults}
\end{figure}

% Other Bayes net examples??

\subsection{QMR-DT}
\label{sec:results_qmr}

We next consider a more complicated Bayesian network model based on the QMR-DT medical diagnosis network~\cite{QMR}. QMR-DT is a bipartite graphical model with one layer of nodes corresponding to latent causes (e.g. diseases, in the medical setting) and a second layer of observed effects (e.g. symptoms). All nodes are binary (i.e. Bernoulli), and the cause nodes are connected to the effects via directed noisy-or links. Appendix~\ref{sec:appendix_code:qmr} shows our implementation.

Our amortized guide program for this model uses a neural network to jointly predict the probabilities of all latent cause variables given a set of observed effects. Since the QMR-DT model contains a large number of discrete random variables, we expect the variance reduction strategies introduced in Section~\ref{sec:optimization:unifiedEstimator} to have significant effect.
Thus, we consider training this guide program with no variance reduction (\emph{Amortized}, step size $10^{-5}$), with per-choice likelihood ratio weights, (\emph{+ local weights}, step size $10^{-3}$), and with both per-choice weights and baselines (\emph{+ baselines}, step size $10^{-2}$). As a point of reference, we also include a mean field model (step size $10^{-2}$) which uses all variance reduction strategies.
Data for our experiments is sampled from a randomly-generated graph with 200 causes and 100 effects. We sample 1000 observations for the training set and an additional 100 for a held-out test set.
% \ndg{note that the fully-marginalized approach isn't really tractable anymore due to the size of the latent space?}

Figure~\ref{fig:qmrResults} shows the results of our experiments. The left plot shows optimization progress under each condition. Without any of the variance reduction strategies, gradients are extremely noisy and optimization makes almost no progress. Using local, per-variable likelihood ratio weights allows optimization to mke progress, and adding per-variable baselines further boosts performance. Though it uses all variance reduction strategies, the mean field model trains significantly more slowly than the variance-reduced amortized models. This happens because the mean field model has separate parameters for each training observation, rather than a single parameter set shared by a neural network, i.e. it has many more parameters that are each updated by very few gradient steps. Amortization thus both facilitates fast posterior prediction and exhibits faster training due to parameter sharing.

We next evaluate the guide program's posterior prediction ability. We use the learned guide to sample latent causes given an observed set of effects from the test set, sample effects given those causes, and then record what percentage of active effects in the test set observation are correctly predicted by the effects `hallucinated' from our model. Specifically, if $\vecstyle{e}$ is a vector of effect variables of length $N$, then the metric we use is:
%%%
\begin{align*}
&\min(F(\vecstyle{e}_{\text{true}}, \vecstyle{e}_{\text{sampled}}), F(\vecstyle{e}_{\text{sampled}}, \vecstyle{e}_{\text{true}}))\\
&F(\vecstyle{e}_1, \vecstyle{e}_2) = \frac{1}{\sum_i^N \vecstyle{e}_1(i)} \sum_{i, \vecstyle{e}_1(i) = 1} \indicator{\vecstyle{e}_2(i) = 1}
\end{align*}
%%%
where $\vecstyle{e}_{\text{true}}$ are effects from the test set and $\vecstyle{e}_{\text{sampled}}$ are hallucinated from the model.
Figure~\ref{fig:qmrResults} Right plots the average $F$ score over 100 runs, where we compare our amortized guide program against using the prior program to sample latent causes. The learned guide program correctly predicts more than twice as many active effects.

\begin{figure}[!ht]
\begin{minipage}{0.6\linewidth}
\centering
\includegraphics[width=\linewidth]{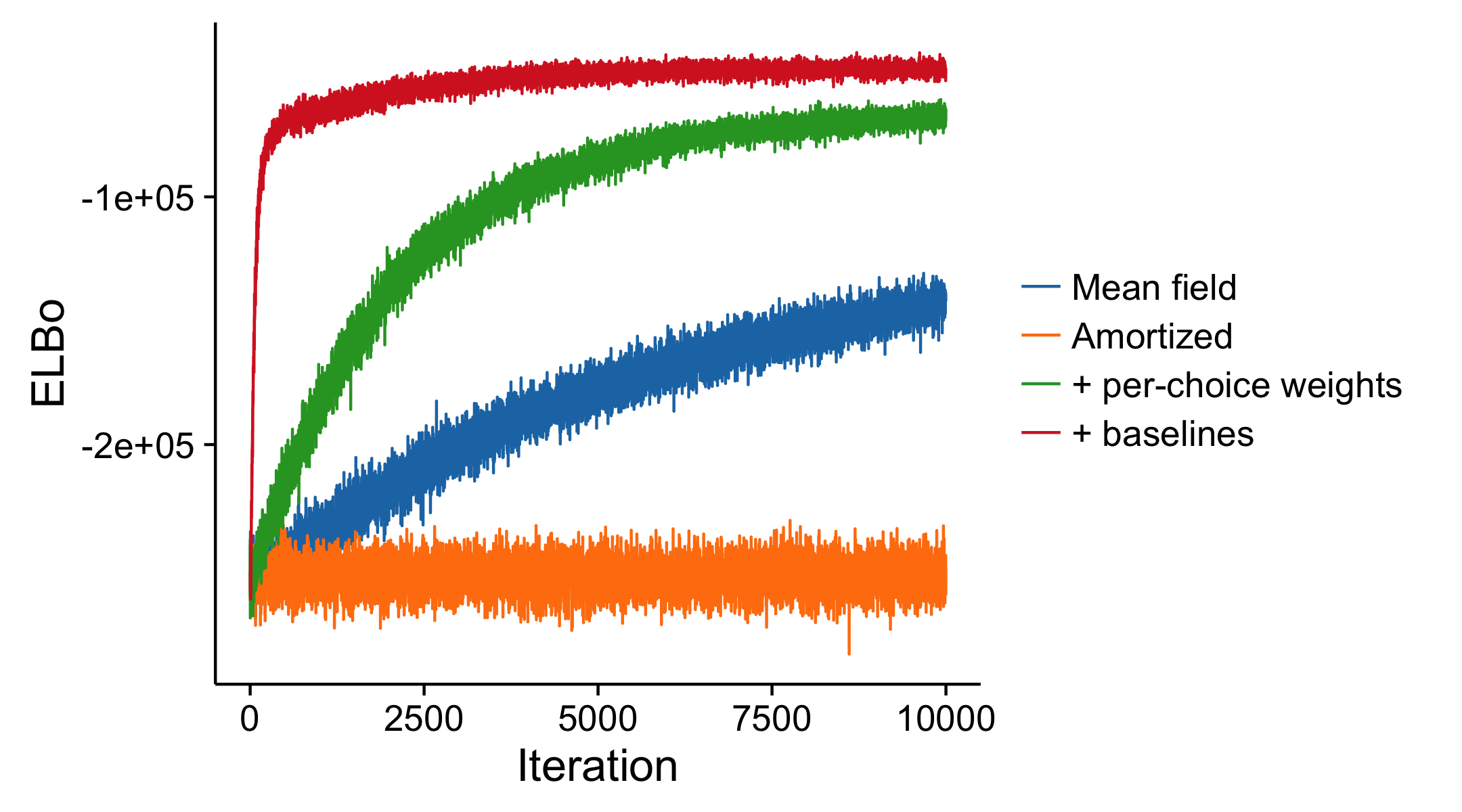}
\end{minipage}
\begin{minipage}{0.4\linewidth}
\centering
\includegraphics[width=\linewidth]{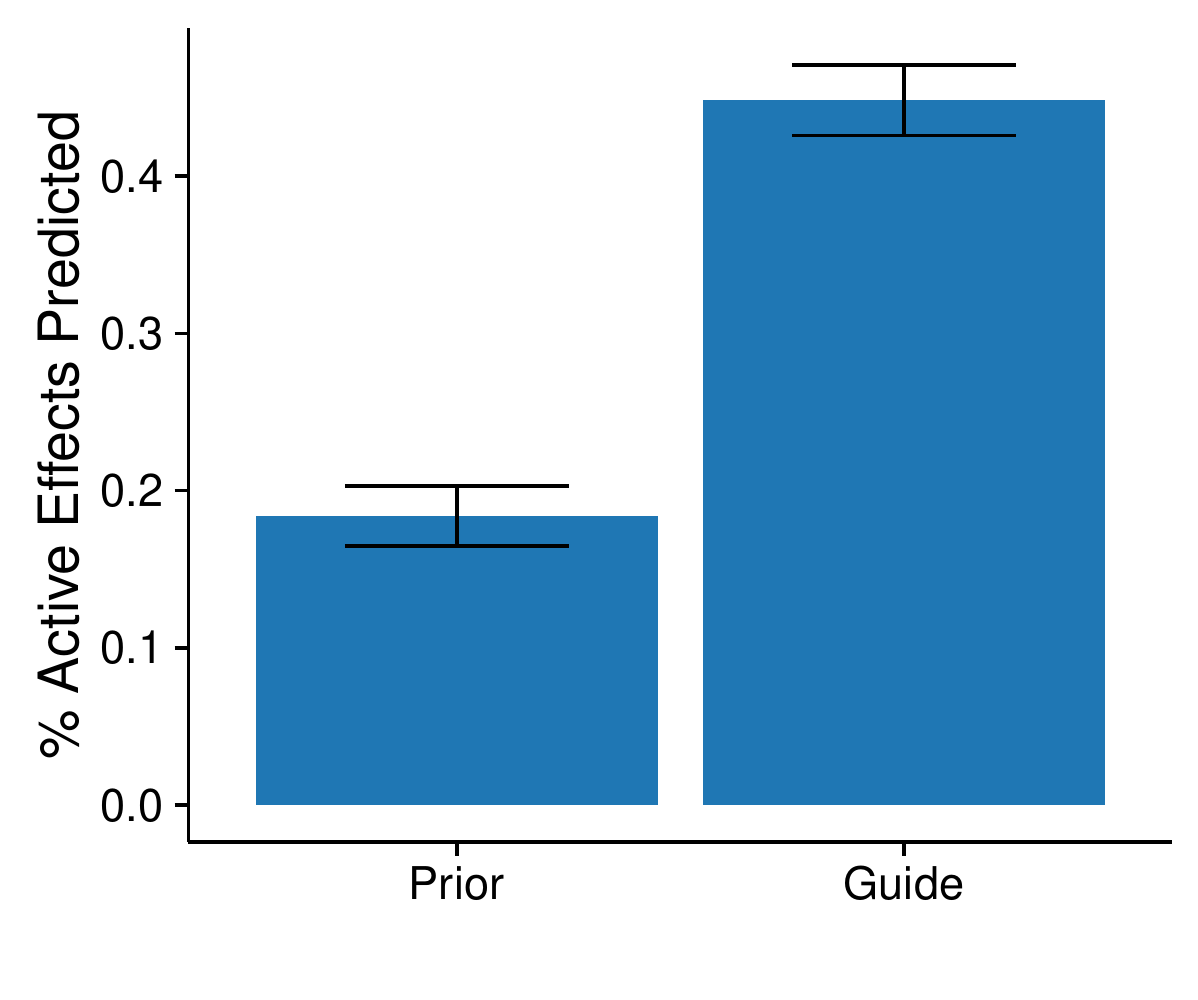}
\end{minipage}
\caption{Performance of a QMR-DT model. \emph{(Left)} ELBo optimization progress during training. \emph{(Right)} Percentage of test set active effects correctly predicted using latent causes sampled from either the prior or the guide program.}
\label{fig:qmrResults}
\end{figure}

In addition to the guide program described above, which predicts all latent causes jointly given the observed effects, we also experimented with `factored' guide programs which predict each latent cause one-by-one given the observed effects. We consider a guide that predicts each latent cause independently (\emph{Factored}), as well as a guide that introduces dependencies between all latent causes via a recurrent neural network (\emph{Factored + GRU}). The recurrent network receives as input the value of each latent cause as it is sampled, maintaining a persistent hidden state  that, in theory, can capture the values of all latent causes sampled thus far. We use the gated recurrent unit (GRU) architecture for its ability to capture a longer-range dependencies~\cite{GRU}, with a hidden state of dimension 20. The code for these programs is shown in Appendix~\ref{sec:appendix_code:qmr}.
We use stepsize 0.01 for both during optimization.

Figure~\ref{fig:qmrResults_factored} compares these guide programs with the joint guide used earlier in this section.
While the independent factored guide performs slightly less well than the joint guide, adding the recurrent neural network to capture posterior dependencies improves performance to slightly better than the joint guide.
One caveat is that the \emph{Factored + GRU} guide takes significantly longer to train in our implementation due to the recurrent network computation.
Later in the paper, we discuss how this guide structure might point the way toward automatically deriving guide programs.

\begin{figure}[!ht]
\begin{minipage}{0.5\linewidth}
\centering
\includegraphics[width=\linewidth]{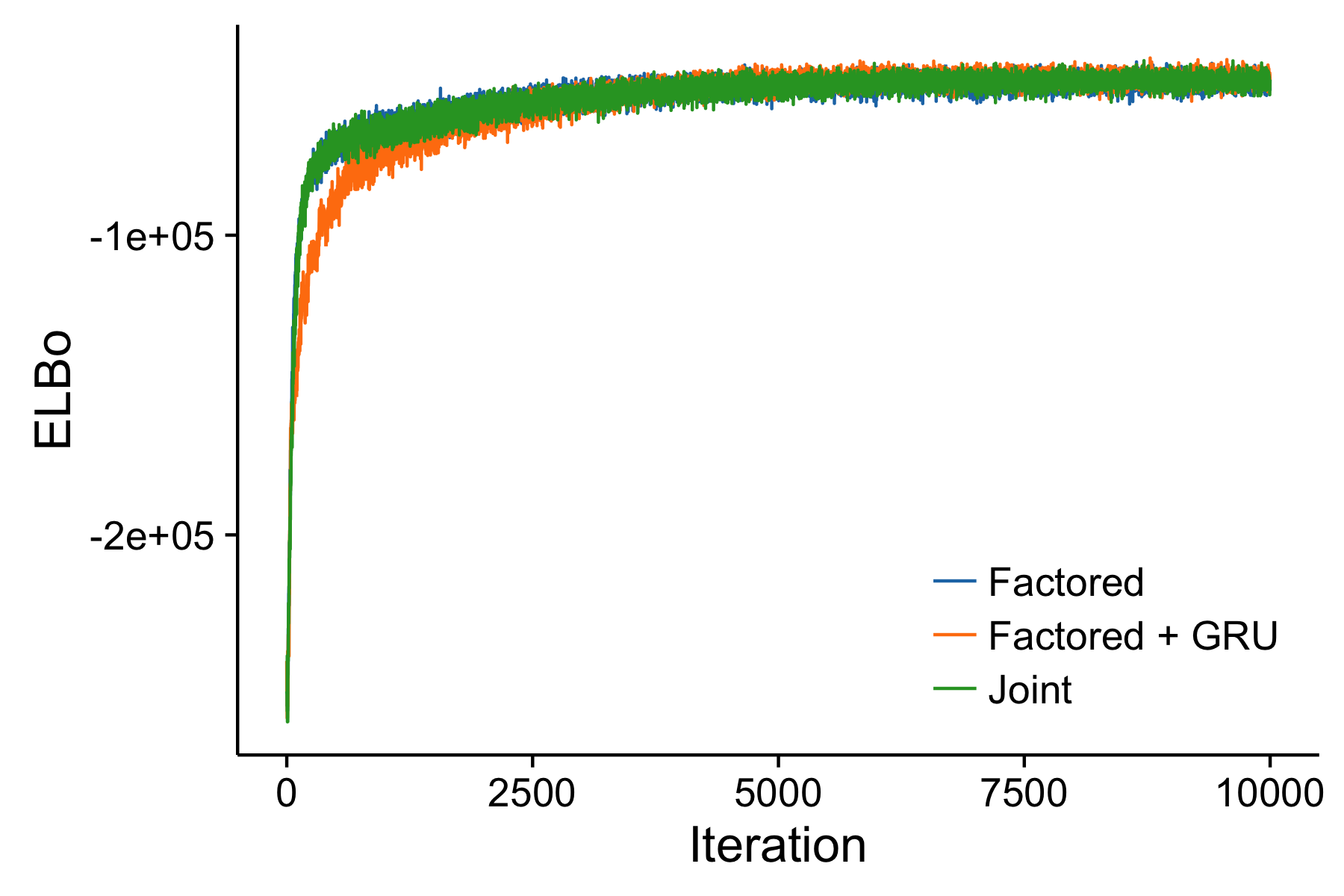}
\end{minipage}
\begin{minipage}{0.5\linewidth}
\centering
\includegraphics[width=\linewidth]{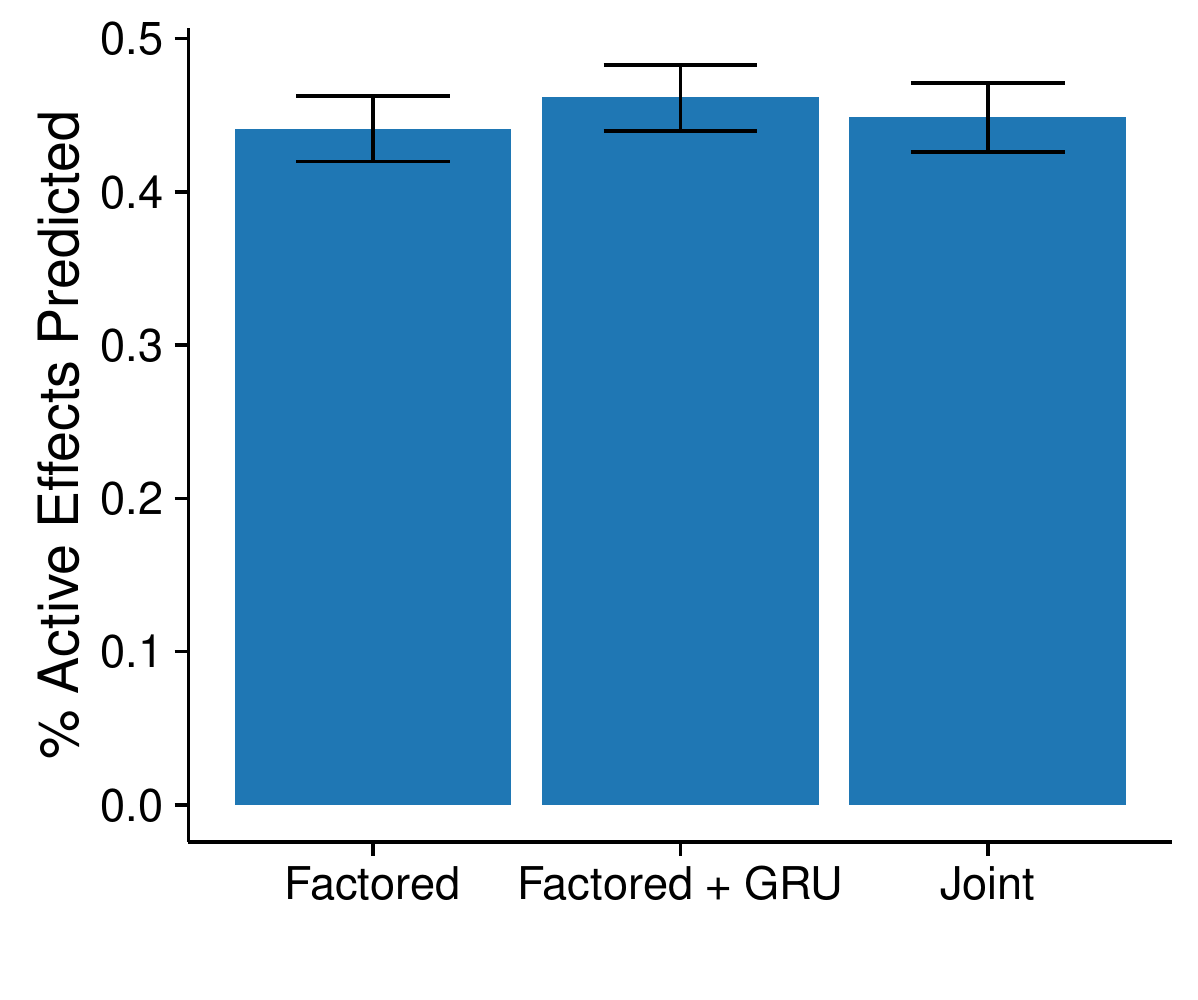}
\end{minipage}
\caption{Experimenting with factored guide programs for QMR-DT. \emph{(Left)} ELBo optimization progress during training. \emph{(Right)} Percentage of test set active effects correctly predicted using latent causes sampled from either the prior or the guide program.}
\label{fig:qmrResults_factored}
\end{figure}

%\ndg{TODO: i think we want to consider accuracy at predicting latents (here causes) on a held out set. this is where posterior correlations should really show up.}
%
%\ndg{TODO: evaluate using addresses to tell the update and predict nets where they are in execution. show results.}
%
%\ndg{TODO: it'd be great to show an example of how daipp can handle missing evidence. we can try to run an rnn over the observations (with address so we know which is which) to build an 'observation vector' then predict the latents from this.}

\subsection{Latent Dirichlet Allocation}
\label{sec:results_lda}

We also used our system to implement amortized inference in Latent Dirichlet Allocation topic models, over a data set of abstracts taken from the Stanford Computation and Cognition Lab's publication page.\footnote{Thanks to Robert Hawkins for creating this dataset.}

We experimented with two different amortized guide programs.
The first is local to each word in a document (\emph{Word-level guide}): it learns to predict the latent topic for each word, given the word and the latent topic distribution for the document.
The second is local to each document (\emph{Document-level guide}): it learns to predict the latent topic distribution for the document, given the words in the document and the latent word distributions for each topic.
These two guides support amortized inference at different granularities and thus have different parameter sharing characteristics, which may lead to different learning behavior.
For comparison, we also included two non-amortized conditions: a mean field model, and a mean field model with the latent choice of topic per word marginalized out (\emph{Marginalized mean field}).
Code for all of these programs can be found in Appendix~\ref{sec:appendix_code:lda}.
We use five topics and learning rate 0.01 in all experiments.

Figure~\ref{sec:results_lda} shows the results of these experiments.
In the optimization progress plot on the left, we see that the marginalized mean field model achieve the highest ELBo.
This is consistent with the results from the Gaussian mixture model experiments: marginalizing out latent variables when possible leads to a tighter bound on the marginal log likelihood.
% In the ELBo progress plot on the left, we see that the marginalized mean field model converges fastest: this is consistent with the result from the Gaussian mixture model experiments and further emphasizes the performance benefits of eliminating discrete random choices when possible.

Of our two amortized guides, the word-level guide performs better (and nearly as well as the marginalized model), likely due to increased parameter sharing. The document-level guide performs at least as well as the mean field model while also being able to efficiently predict topic distributions for previously-unseen documents.

Figure~\ref{sec:results_lda} Right shows the top ten highest probability words in each inferred topic for the marginalized mean field model. From left to right, these topics appear to be about experiments, pragmatic language models, knowledge acquisition, probabilistic programming languages, and a grab bag of remaining topics.

\begin{figure}[!ht]
\begin{minipage}{0.5\linewidth}
\centering
\includegraphics[width=\linewidth]{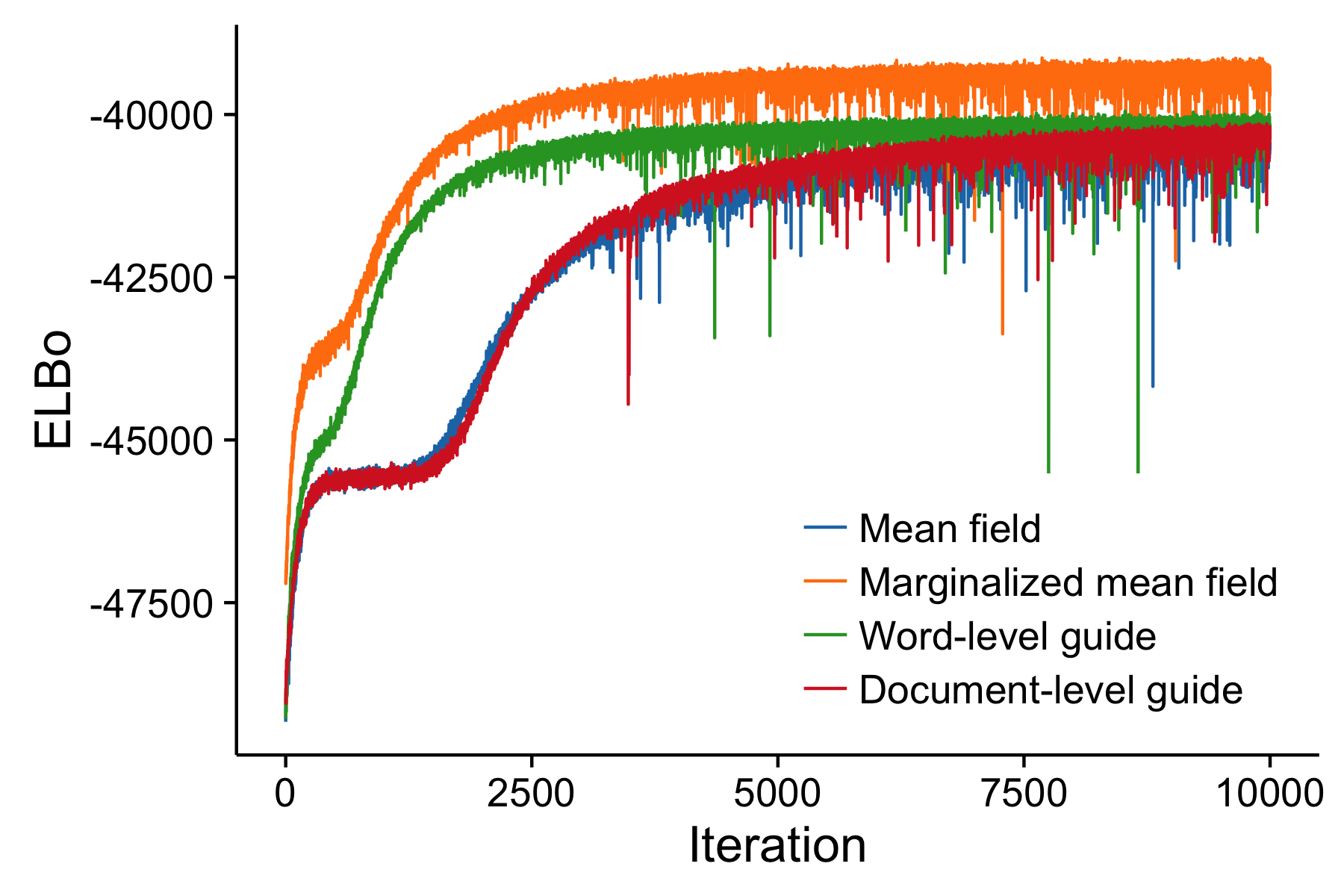}
\end{minipage}
\begin{minipage}{0.5\linewidth}
\tiny
\centering
\begin{tabular}{c c c c c}
\textbf{Topic 1} & \textbf{Topic 2} & \textbf{Topic 3} & \textbf{Topic 4} & \textbf{Topic 5}\\
model & model & loop & probabilistic & attempt\\
causal & language & learning & models & inference\\
dynamic & pragmatic & abstract & programming & novel\\
judgments & abstract & extension & church & decisions\\
evidence & interpretation & bayesian & abstract & required\\
abstract & inference & knowledge & using & models\\
actions & knowledge & human & inference & behavior\\
pedagogical & words & causal & stochastic & bayesian\\
people & form & theories & language & cognition\\
participants & using & theory & model & expressive
\end{tabular}
\end{minipage}
\caption{Performance of Latent Dirichlet Allocation models. \emph{(Left)} ELBo optimization progress during training. The optimization objective for the marginalized model is a tighter bound on the marginal log likelihood and thus has a higher asymptote. \emph{(Right)} Top ten highest probability words in each inferred topic for the best-performing model.}
\label{fig:ldaResults}
\end{figure}

\subsection{Neural Generative Models: Variational Autoencoder \& Sigmoid Belief Network}
\label{sec:results_vae}

Our system naturally supports generative models which use neural network components. Two prominent examples of models in this class include the Variational Autoencoder (VAE) and Sigmoid Belief Networks (SBN). Both models sample latent variables from a multivariate distribution and then transform the result via a neural network to produce observed variables, often in the form of an image. The VAE uses a latent multivariate Gaussian distribution, whereas the SBN uses a latent multivariate Bernoulli.

Appendix~\ref{sec:appendix_code} shows implementations of these models in our system. Our VAE implementation follows the original description of the model by Kingma and Welling~\cite{AEVB}, and our SBN implementation follows that of Mnih and Gregor~\cite{NVIL}.
The VAE uses a 20-dimensional latent code, and the SBN uses a single layer of 200 hidden variables. Our system cannot express the two-layer SBN of Mnih and Gregor, because its guide model samples the latent variables in the reverse order of the generative model.

Figure~\ref{fig:vae_sbn_results} Left shows results of training these models on the MNIST dataset, using Adam with a step size of 0.001.
While both models train quickly at first, the SBN's training slows more noticeably than the VAE's due to its discrete nature. It takes more than three times as many iterations to achieve the same ELBo.
In Figure~\ref{fig:vae_sbn_results} Right, we qualitatively evaluate both models by using them to reconstruct images from the MNIST test set. We use the guide program to sample latent variables conditional on the images in the ``Target'' column (i.e. the `encoder' phase). We then transform these latent variables using the generative model's neural networks (i.e. the `decoder' phase) to produce the reconstructed images in the ``VAE'' and ``SBN'' columns.
As suggested by their training behavior, the VAE is able to generate higher-quality reconstructions after less training.

Our optimization exhibits some differences from the previous work.
For the VAE, Kingma and Welling~\cite{AEVB} exploit the closed-form solution of the KL divergence between two Gaussians to create an even lower-variance estimator of the ELBo gradient. We use a more general formulation, but our system can still successfully train the model.
For the SBN, Mnih and Gregor~\cite{NVIL} use neural networks to compute the per-variable baselines $b_i$ in Equation~\ref{eq:finalEstimator}, whereas we use a simpler approach (see Appendix~\ref{sec:appendix_proofs}).
However, the key point is that each of these models was described in a simple WebPPL program with neural guides and optimized by the default system, without the need for additional implementation efforts.

\begin{figure}
\begin{minipage}{0.5\linewidth}
\centering
\includegraphics[width=\linewidth]{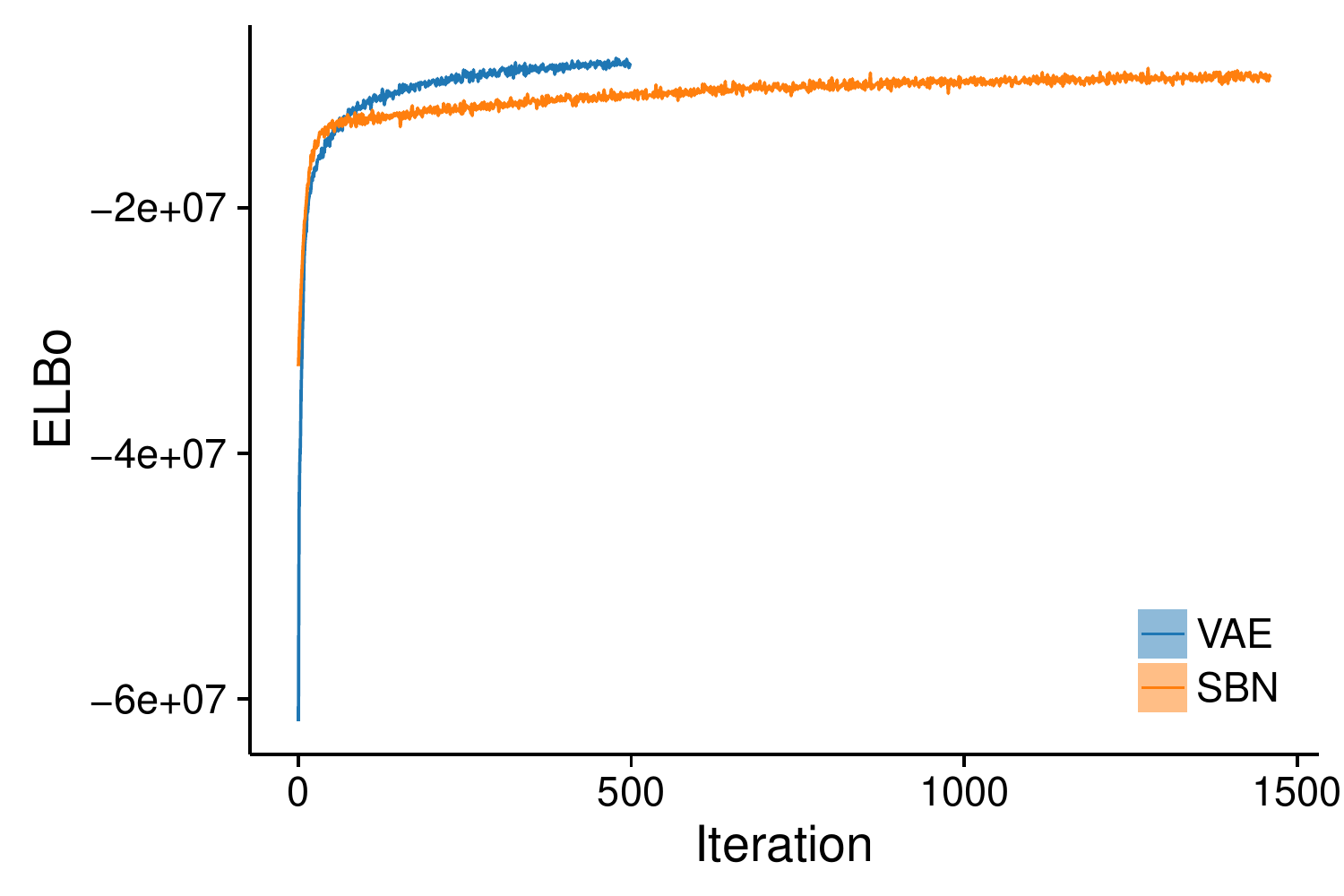}
\end{minipage}
\hspace{2em}
\begin{minipage}{0.5\linewidth}
\setlength{\tabcolsep}{1pt}
\centering
\begin{tabular}{c  c c c c c c}
Target & \multicolumn{3}{c}{VAE} & \multicolumn{3}{c}{SBN}
\\
 \includegraphics[width=0.12\linewidth]{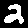}
 \hspace{3pt}
& \includegraphics[width=0.12\linewidth]{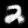}
& \includegraphics[width=0.12\linewidth]{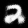}
& \includegraphics[width=0.12\linewidth]{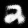}
\hspace{3pt}
& \includegraphics[width=0.12\linewidth]{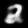}
& \includegraphics[width=0.12\linewidth]{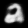}
& \includegraphics[width=0.12\linewidth]{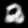}
\\
 \includegraphics[width=0.12\linewidth]{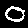}
 \hspace{3pt}
& \includegraphics[width=0.12\linewidth]{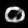}
& \includegraphics[width=0.12\linewidth]{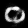}
& \includegraphics[width=0.12\linewidth]{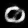}
\hspace{3pt}
& \includegraphics[width=0.12\linewidth]{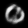}
& \includegraphics[width=0.12\linewidth]{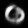}
& \includegraphics[width=0.12\linewidth]{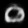}
\\
 \includegraphics[width=0.12\linewidth]{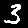}
 \hspace{3pt}
& \includegraphics[width=0.12\linewidth]{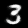}
& \includegraphics[width=0.12\linewidth]{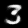}
& \includegraphics[width=0.12\linewidth]{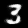}
\hspace{3pt}
& \includegraphics[width=0.12\linewidth]{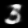}
& \includegraphics[width=0.12\linewidth]{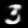}
& \includegraphics[width=0.12\linewidth]{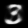}
\\
 \includegraphics[width=0.12\linewidth]{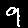}
 \hspace{3pt}
& \includegraphics[width=0.12\linewidth]{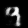}
& \includegraphics[width=0.12\linewidth]{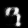}
& \includegraphics[width=0.12\linewidth]{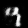}
\hspace{3pt}
& \includegraphics[width=0.12\linewidth]{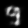}
& \includegraphics[width=0.12\linewidth]{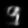}
& \includegraphics[width=0.12\linewidth]{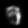}
\end{tabular}
\end{minipage}
\caption{Evaluating the Variational Autoencoder (VAE) and Sigmoid Belief Network (SBN) programs on the MNIST dataset. \emph{(Left)} ELBo optimization progress during training. \emph{(Right)} Reconstructing the images in the ``Target'' column using both models.}
\label{fig:vae_sbn_results}
\end{figure}

%Things to try if we have time:
%\begin{itemize}
%\item{Simple bayesian NN (show how easy it is to do full variational Bayes)}
%\item{HMM-type families: Deep kalman filter? Or other RNN VAE type thing?}
%\item{PCFG (prob won’t work, but useful to understand). Continuous feature-passing version?}
%\end{itemize}
%

%auto-ignore
\section{Deriving Guides Automatically}
\label{sec:autoGuide}

Thus far, we have shown how we can succesfully create and train guide programs for several types of generative models. However, writing guide programs can sometimes be tedious and repetitive; for example, note the large amount of shared structure between the guides shown in Figures~\ref{fig:bn_oneLatent} and \ref{fig:bn_twoLatent}. Furthermore, it is not always obvious how to write a good guide program. In Figure~\ref{fig:bn_twoLatent}, knowledge of the structure of this very simple generative model led us to add a direct dependency between the two latent variables in the guide. For general programs---especially large, complex ones---it will not always be clear what these dependencies are or how to capture them with a guide.

This section describes our early experience with automatically deriving guide programs. We first describe how our system provides sensible default behavior that can make writing some guides less cumbersome. We then outline how the system might be extended to automatically derive guides for any program using recurrent neural networks.

\subsection{Mean Field by Default}
\label{sec:autoGuide:meanField}

If a call to \ic{sample} is not provided with an explicit guide distribution, our system automatically inserts a mean field guide. For example, the code \ic{sample(Gaussian(\{mu: 0, sigma: 1\}))} results in:
%%%
\begin{lstlisting}
sample(Gaussian({mu: 0, sigma: 1}), {
   guide: Gaussian({mu: param({name: <@\hilite{<auto_name>}@>}), sigma: softplus(param({name: <@\hilite{<auto_name>}@>}))})
})
\end{lstlisting}
%%%
where parameter bounding transforms such as \ic{softplus} are applied based on bounds metadata provided with each primitive distribution type. We use reparameterizable guides for continuous distributions (see Appendix~\ref{sec:appendix_reparam}).

Since this process declares new optimizable parameters automatically, we must automatically generate names for these parameters. Our system names parameters according to where they are declared in the program execution trace, using the same naming technique as is used for random choices in probabilistic programming MCMC engines~\cite{Lightweight}. Since the names of these parameters are tied to the structure of the program, they cannot be re-used by other programs (as in the `Further Optimization' example of Section~\ref{sec:furtherOptim}).

\subsection{Beyond Mean Field: Automatic Factored Guides with Recurrent Networks}

In Section~\ref{sec:results_qmr}, we experimented with a factored guide program for the QMR–DT model. We think that this general style of guide---predicting each random choice in sequence, conditional on the hidden state of a recurrent neural network---might be generalized to an automatic guide for any program, as any probabilistic program can be decomposed into a sequence of random choices. In our QMR-DT experiments, we used a separate neural network (with separate parameters) to predict each latent variable (i.e. random choice). For complex models and large data sets, this approach would lead to a computationally unfeasible explosion in the number of parameters. Furthermore, it is likely that the prediction computations for many random choices in the program are related. For example, in the QMR-DT program, latent causes that share many dependent effects may be well predicted by the same or very similar networks.

Given these insights, we imagine a universally-applicable guide that uses a single prediction network for all random choices, but to which each random choice provides an additional identifying input. These IDs should be elements in a vector space, such that more `similar' random choices have IDs which are close to one another for some distance metric in the vector space. One possible way to obtain such IDs would be to learn an embedding of the program-structural addresses of each random choice~\cite{Lightweight}. These might be learned in an end-to-end fashion by making them learnable parameter vectors in the overall variational optimization (i.e. letting closeness in the embedding space be an emergent property of optimizing our overall objective).

% \subsection{Capturing Posterior Dependencies with Context Nets}

% Context nets, val2vec, and vec2dist.

% When to init context (e.g. incorporating datum in each mapData iteration vs. something more global).

% Architecture choices: LSTM, bilinear resnets, linear predict net, etc.

% If we have time:
% \begin{itemize}
% \item{Context nets with long distance connections?}
% \item{Automatic mixtures for continuous distributions?}
% \item{Other interesting guides?}
% \end{itemize}

%auto-ignore
\section{Conclusion}
\label{sec:conclusion}

In this paper, we presented a system for amortized inference in probabilistic programs.
Amortization is achieved through parameterized \emph{guide programs} which mirror the structure of the original program but can be trained to approximately sample from the posterior.
We introduced an interface for specifying guide programs which is flexible enough to reproduce state-of-the-art variational inference methods.
We also demonstrated how this interface supports model learning in addition to amortized inference.
We developed and proved the correctness of an optimization method for training guide programs, and we evaluated its ability to optimize guides for Bayesian networks, topic models, and deep generative models.

\subsection{Future Work}
\label{sec:conclusion_futureWork}

There are many exciting directions of future work to pursue in improving amortized inference for probabilistic programs. The system we have presented in this paper provides a platform from which to explore these and other possibilities:

\paragraph{More modeling paradigms}
In this paper, we focused on the common machine learning modeling paradigm in which a global generative model generates many IID data points.
There are many other modeling paradigms to consider.
For example, time series data is common in machine learning applications. Just as we developed \ic{mapData} to facilitate efficient inference in IID data models, we might develop an analogous data processing function for time series data (i.e. \ic{foldData}). Using neural guides with such a setup would permit amortized inference in models such as Deep Kalman Filters~\cite{DeepKalmanFilters}.
In computer vision and computer graphics, a common paradigm for generative image models is to factor image generation into multiple steps and condition each step on the partially-generated image thus far~\cite{PixelCNN,NGPM}.
Such `yield-so-far' models should also be possible to implement in our system.

\paragraph{Better gradient estimators}
While the variance reduction strategies employed by our optimizer make inference with discrete variables tractable, it is still noticeably less efficient then with purely continuous models.
Fortunately, there are ongoing efforts to develop better, general-purpose discrete estimators for stochastic gradients~\cite{MuProp,VIMCO}.
It should be possible to adapt these methods for probabilistic programs.

\paragraph{Automatic guides}
As discussed in Section~\ref{sec:autoGuide}, we believe that automatically deriving guide programs using recurrent neural networks may soon be possible.
Recent enhancements to recurrent networks may be necessary to make this a reality.
For example, the external memory of the Neural Turing Machine may be better at capturing certain long-range posterior dependencies~\cite{NTM}.
We might also draw inspiration from the Neural Programmer-Interpreter~\cite{NPI}, whose stack of recurrent networks which communicate via arguments might better capture the posterior dataflow of arbitrary programs.

\paragraph{Other learning objectives}
In this paper, we focused on optimizing the ELBo.
If we flip the direction of KL divergence in Equation~\ref{eq:elbo}, the resulting functional is an \emph{upper} bound on the log marginal likelihood of the data---an `Evidence Upper Bound,' or EUBo.
Computing the EUBo and its gradient requires samples from the true posterior and is thus unusable in many applications, where the entire goal of amortized inference is to find a way to tractably generate such samples.
However, some applications can benefit from it, if the goal is to speed up an existing tractable inference algorithm (e.g. SMC~\cite{NGPM}), or if posterior execution traces are available through some other means (e.g. input examples from the user).
There may also be less extreme ways to exploit this idea for learning.
For example, in a \ic{mapData}-style program, we might interleave normal ELBo updates with steps that hallucinate data from the posterior predictive (using a guide for global model parameters) and train the local guide to correctly parse these `dreamed-up' examples. Such a scheme bears resemblance to the wake-sleep algorithm~\cite{WakeSleep}.

\paragraph{Control flow}
While our system's one-to-one mapping between random choices in the guide and in the target program makes the definition and analysis of guides simple, there are scenarios in which more flexibility is useful.
In some cases, one may want to insert random choices into the guide which do not occur in the target program (e.g. using a compound distribution, such as a mixture distribution, as a guide).
And for models in which there is a natural hierarchy between the latent variables and the observed variables, having the guide run `backwards' from the observed variables to the top-most latents has been shown to be useful~\cite{StochasticInverses,NeuralStochasticInverses,NVIL}.
It is worth exploring how to support these (and possibly even more general) control flow deviations in a general-purpose probabilistic programming inference system.

%auto-ignore

\section*{Acknowledgments}

This material is based on research sponsored by DARPA under agreement number FA8750-14-2-0009. The U.S. Government is authorized to reproduce and distribute reprints for Governmental purposes notwithstanding any copyright notation thereon. The views and conclusions contained herein are those of the authors and should not be interpreted as necessarily representing the official policies or endorsements, either expressed or implied, of DARPA or the U.S. Government.

% \small
\bibliographystyle{plain}
\bibliography{main}

\appendix
%auto-ignore

\section{Appendix: Reparameterizations}
\label{sec:appendix_reparam}

Examples of primitive random choice distributions that can be reparameterized via a location-scale transform:

\begin{center}
\renewcommand{\arraystretch}{1.5}
\begin{tabular}{c | c | c}
\textbf{Distribution} & $\bm{\reparamVar \sim \reparamDist(\cdot)}$ & $\bm{\reparamXform(\reparamVar)}$ \\
\hline
Gaussian($\mu$, $\sigma$) & Gaussian($0$, $1$) & $\mu + \sigma \cdot \reparamVar$ \\
LogitNormal($\mu$, $\sigma$) & Gaussian($0$, $1$) & $\text{sigmoid}(\mu + \sigma \cdot \reparamVar)$ \\
LogisticNormal($\bm{\mu}$, $\bm{\sigma}$) & Gaussian($\bm{0}$, $\bm{1}$) & $\text{simplex}(\bm{\mu} + \bm{\sigma} \cdot \reparamVar)$ \\
InverseSoftplusNormal($\mu$, $\sigma$) & Gaussian($0$, $1$) & $\text{softplus}(\mu + \sigma \cdot \reparamVar)$ \\
Exponential($\lambda$) & Uniform($0$, $1$) & $-\log(\reparamVar) / \lambda$ \\
Cauchy($x_0$, $\gamma$) & Uniform($0$, $1$) & $x_0 + \gamma \cdot \tan( \pi \cdot (\reparamVar - 0.5) )$
\end{tabular}
\end{center}

Examples of primitive distributions that do not have a location-scale transform but can be guided by a reparameterizable approximating distribution:

\begin{center}
\renewcommand{\arraystretch}{1.5}
\begin{tabular}{c | c}
\textbf{Distribution} & \textbf{Guide Distribution} \\
\hline
Uniform & LogitNormal \\
Beta & LogitNormal \\
Gamma & InverseSoftplusNormal \\
Dirichlet & LogisticNormal
\end{tabular}
\end{center}

%auto-ignore

\section{Appendix: Gradient Estimator Derivations \& Correctness Proofs}
\label{sec:appendix_proofs}

\newtheorem{lemma}{Lemma}
\newtheorem{theorem}{Theorem}

\subsection{Derivation of Unified Gradient Estimator (Equation~\ref{eq:hybridEstimator})}
\label{sec:appendix:estDerivation}

\begin{align}
\gradparams \elbo
&= \gradparams \expect_\reparamDist [ \log p(\xformedVars, \observedVars) - \log \guide(\xformedVars | \observedVars) ]\nonumber\\
&= \gradparams \int_\reparamVars \reparamDist(\reparamVars | \observedVars) ( \log p(\xformedVars, \observedVars) - \log \guide(\xformedVars | \observedVars) )\nonumber\\
&= \int_\reparamVars \gradparams \reparamDist(\reparamVars | \observedVars) ( \log p(\xformedVars, \observedVars) - \log \guide(\xformedVars | \observedVars) ) + \reparamDist(\reparamVars | \observedVars) \gradparams ( \log p(\xformedVars, \observedVars) - \log \guide(\xformedVars | \observedVars) )\nonumber \\
\label{eq:estDerivation_trick}
&= \int_\reparamVars \reparamDist(\reparamVars | \observedVars) \gradparams \log \reparamDist(\reparamVars | \observedVars) ( \log p(\xformedVars, \observedVars) - \log \guide(\xformedVars | \observedVars) ) + \reparamDist(\reparamVars | \observedVars) \gradparams ( \log p(\xformedVars, \observedVars) - \log \guide(\xformedVars | \observedVars) ) \\
&= \expect_\reparamDist [ \gradparams \log \reparamDist(\reparamVars | \observedVars) ( \log p(\xformedVars, \observedVars) - \log \guide(\xformedVars | \observedVars) ) + \gradparams( \log p(\xformedVars, \observedVars) - \log \guide(\xformedVars | \observedVars) )]\nonumber
\end{align}
Line~\ref{eq:estDerivation_trick} makes use of the identity $\nabla f(x) = f(x) \nabla \log f(x)$.

\subsection{Zero Expectation Identities}
\label{sec:appendix:zeroexp}

In what follows, we will make frequent use of the following:

\begin{lemma}
If $f(x)$ is a probability distribution, then:

\begin{equation*}
\expect_f[\nabla \log f(x)] = 0
\end{equation*}
\label{lem:zeroexp}
\end{lemma}
%%%
\begin{proof}
\begin{equation*}
\expect_f[\nabla \log f(x)]
= \int_x f(x) \nabla \log f(x)
= \int_x \nabla f(x)
= \nabla \int_x f(x)
= \nabla 1
= 0
\end{equation*}
\end{proof}

\begin{lemma}
For a discrete random choice $\reparamVar_i$ and a function $f(\reparamVars_{<i}, \observedVars)$:
\begin{equation*}
\expect_\reparamDist [ \gradparams \log \guide(\reparamXform_i(\reparamVar_i) | \reparamXform(\reparamVars_{<i}), \observedVars) f(\reparamVars_{<i}, \observedVars) ] = 0
\end{equation*}
\label{lem:zeroexp2}
\end{lemma}
%%%
\begin{proof}
\begin{align*}
\expect_\reparamDist [ \gradparams \log \guide(\reparamXform_i(\reparamVar_i) | \reparamXform(\reparamVars_{<i}), \observedVars) f(\reparamVars_{<i}, \observedVars) ]
&= \int_\reparamVars \reparamDist(\reparamVars | \observedVars) \gradparams \log \guide(\reparamXform_i(\reparamVar_i) | \reparamXform(\reparamVars_{<i}), \observedVars) f(\reparamVars_{<i}, \observedVars)\\
&= \int_\reparamVars \reparamDist(\reparamVars | \observedVars) \gradparams \log \reparamDist(\reparamVar_i | \reparamVars_{<i}, \observedVars) f(\reparamVars_{<i}, \observedVars)\\
&= \int_{\reparamVars_{<i}} \reparamDist(\reparamVars_{<i} | \observedVars) f(\reparamVars_{<i}, \observedVars) \sum_{\reparamVar_i} \reparamDist(\reparamVar_i | \reparamVars_{<i}, \observedVars) \gradparams \log \reparamDist(\reparamVar_i | \reparamVars_{<i}, \observedVars) \int_{\reparamVars_{>i}} \reparamDist(\reparamVars_{>i} | \reparamVars_{\leq i}, \observedVars)\\
&= \int_{\reparamVars_{<i}} \reparamDist(\reparamVars_{<i} | \observedVars) f(\reparamVars_{<i}, \observedVars) \cdot \expect_{\reparamDist} [ \gradparams \log \reparamDist(\reparamVar_i | \reparamVars_{<i}, \observedVars) ] \cdot 1\\
&= \int_{\reparamVars_{<i}} \reparamDist(\reparamVars_{<i} | \observedVars) f(\reparamVars_{<i}, \observedVars) \cdot 0 = 0
\end{align*}
where the last line makes use of Lemma~\ref{lem:zeroexp}.
\end{proof}

\subsection{Variance Reduction Step 1: Zero Expectation $W$ Terms}

In this section, we show that for each random choice $i$, we can remove terms from $W(\reparamVars, \observedVars)$ to produce $w_i(\reparamVars, \observedVars)$. Specifically, we prove that while $w_i(\reparamVars, \observedVars) \neq W(\reparamVars, \observedVars)$, we still have $\expect_\reparamDist [ \gradparams \log \guide(\reparamXform_i(\reparamVar_i) | \reparamXform(\reparamVars_{<i}), \observedVars) w_i(\reparamVars, \observedVars) ] = \expect_\reparamDist [ \gradparams \log \guide(\reparamXform_i(\reparamVar_i) | \reparamXform(\reparamVars_{<i}), \observedVars) W(\reparamVars, \observedVars) ]$.

To enable the computation of each $w_i$, our system builds a directed acyclic dependency graph as the program executes. The graph is constructed as follows:
%%%
\begin{itemize}
\item{\textbf{On program start:} Create a root node \ic{root}. Set \ic{prev = root}. This holds the previous node and will be used to assign node parents.}
\item{\textbf{On \ic{sample} or \ic{observe}:} Create a new node \ic{node} representing this random choice/observation. Set \ic{node.parent = prev}. Update \ic{prev = node}.}
\item{\textbf{On \ic{mapData} begin:} Create two new nodes \ic{split} and \ic{join}. These nodes will delimit the beginning and ending of the \ic{mapData} iteration. Push \ic{split, join} onto a stack \ic{mapDataStack}. This stack keeps track of \ic{split, join} nodes when there are nested calls to \ic{mapData}.}
\item{\textbf{On \ic{mapData} iteration begin:} Retrieve \ic{split, join = top(mapDataStack)}. Update \ic{prev = split}. This step reflects the independence assumptions of \ic{mapData}: an iteration of \ic{mapData} does not depend on any previous iterations, so there are no such edges in the graph. Instead, each iteration of \ic{mapData} points back to the beginning of the \ic{mapData} call.}
\item{\textbf{On \ic{mapData} iteration end:} Retrieve \ic{split, join = top(mapDataStack)}. Add \ic{prev} to \ic{join.parents}. This step connects the last random choice in each \ic{mapData} iteration to the \ic{join} node.}
\item{\textbf{On \ic{mapData} end:} Retrieve \ic{split, join = top(mapDataStack)}. Update \ic{prev = join}. This step acknowledges that any subsequent computation may depend on the \ic{mapData} as a whole.}
\end{itemize}
%%%
In this graph, all nodes correspond to random choices or observations, except for the special \ic{mapData} nodes \ic{split} and \ic{join}.
When there are no calls to \ic{mapData}, the graph has a linear topology, where each node is connected via a parent edge to the previously-sampled/observed node.
\ic{mapData} calls introduce fanout-fanin subgraphs: the \ic{split} node fans out into separate linear chains for each \ic{mapData} iteration, and the last nodes of these chains fan in to the \ic{join} node. Figure~\ref{fig:graphExample} shows the resulting graph for one execution of a simple example program.

\begin{figure}[!ht]
\begin{minipage}{0.65\linewidth}
\begin{lstlisting}
var data = loadData('data.json');
var model = function() {
   var a = sample(Bernoulli({p: 0.5}));
   mapData({data: data}, function(y) {
      var x = sample(Gaussian({mu: a ? 0 : 5, sigma: 1}));
      observe(Gaussian({mu: x, sigma: 0.5}), y);
   });
   return a;
}
\end{lstlisting}
\end{minipage}
\begin{minipage}{0.35\linewidth}
\centering
\includegraphics[width=\linewidth]{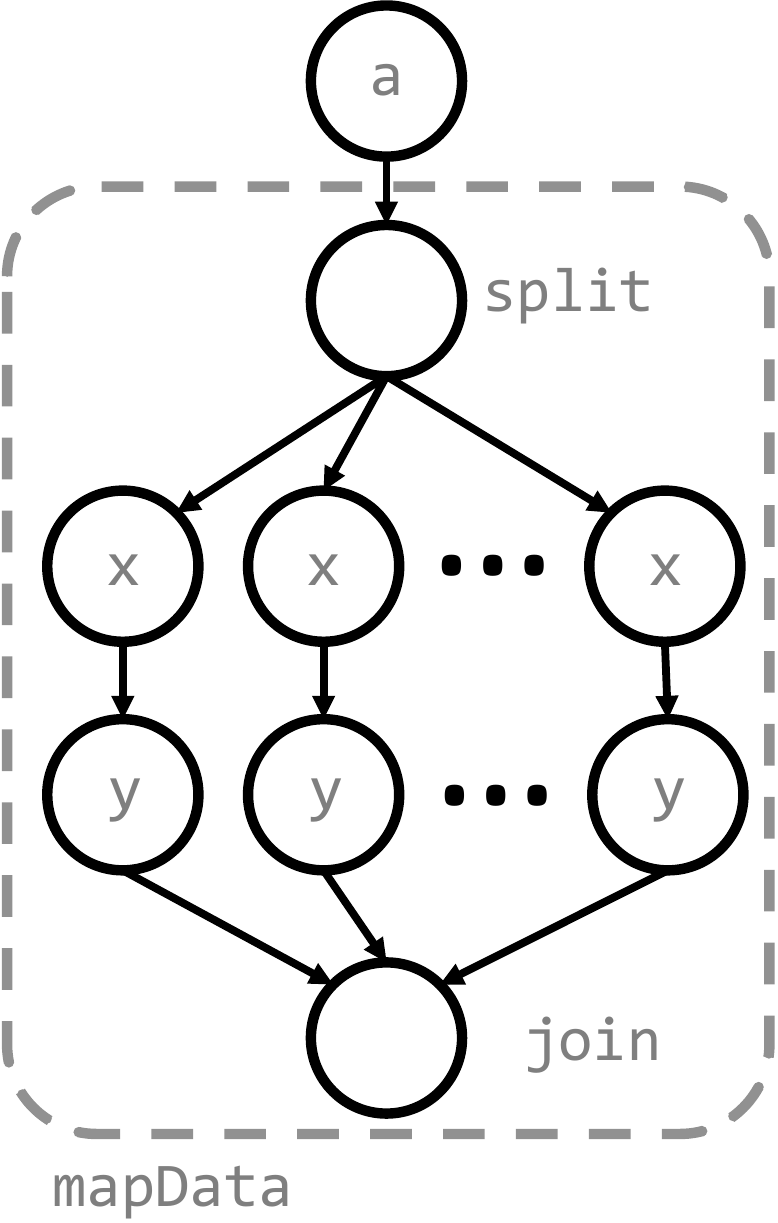}
\end{minipage}
\caption{A conservative dependency graph \emph{(Right)} resulting from one execution through a simple program \emph{(Left)}.}
\label{fig:graphExample}
\end{figure}

By construction, this graph is an overly-conservative dataflow dependency graph: if random choice $\reparamVar_a$ flows into random choice $\reparamVar_b$ (or observation $\observedVar_c$), then a path $\reparamVar_a \to \reparamVar_b$ (or $\reparamVar_a \to \observedVar_c$) exists in the graph. The converse is not necessarily true (i.e. there can exist edges between nodes that have no dataflow dependency). Note also that, by construction, the existence of a path $\reparamVar_a \to \reparamVar_b$ implies that $\reparamVar_b$ was sampled after $\reparamVar_a$ in the program execution order.

From the perpective of a random choice node $\reparamVar_i$, the graph nodes can be partitioned into the following subsets:
%%%
\begin{itemize}
\item{$\mathcal{D}_i$: the nodes ``downstream'' from $\reparamVar_i$ (i.e. the set of all nodes $d$ for which an edge $\reparamVar_i \to d$ exists.}
\item{$\mathcal{U}_i$: the nodes ``upstream'' from $\reparamVar_i$ (i.e the set of all nodes $u$ for which an edge $u \to \reparamVar_i$ exists.}
\item{$\mathcal{C}_i$: the set of nodes which are in neither $\mathcal{D}_i$ nor $\mathcal{U}_i$.}
\end{itemize}
%%%
Figure~\ref{fig:graphPartitions} illustrates these partitions on an example graph. For convenience, we also define $\mathcal{B}_i \equiv \mathcal{U}_i \cup \mathcal{C}_i$.

\begin{figure}[!ht]
\centering
\input{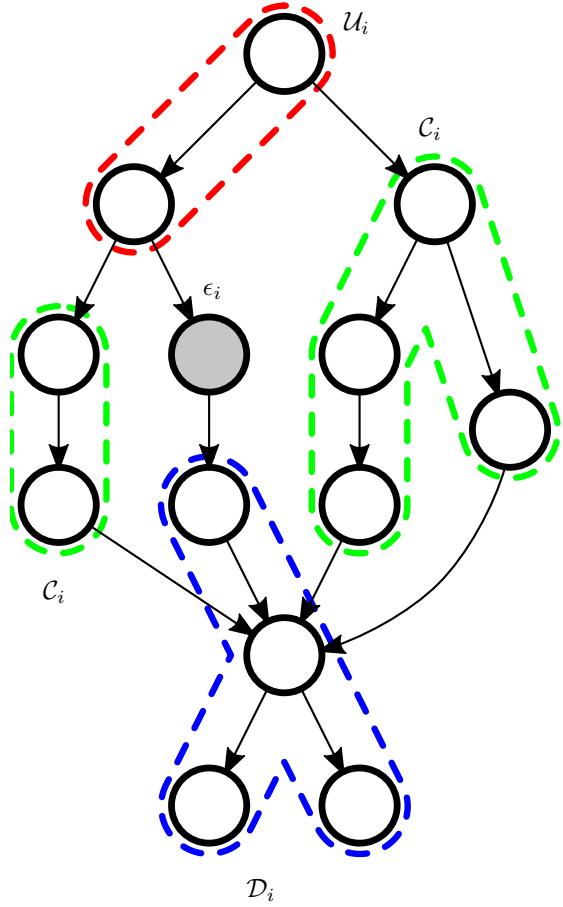}
\caption{Partitioning a dependency graph into multiple node subsets from the perspective of a random variable node $\reparamVar_i$.}
\label{fig:graphPartitions}
\end{figure}

For any given random choice $\reparamVar_i$, these partitions allows us to factor $W(\reparamVars, \observedVars)$ into three terms:
%%%
\begin{align*}
W_i(\reparamVars, \observedVars)
&= \log \frac{p(\xformedVars, \observedVars)}{\guide(\xformedVars | \observedVars)}\\
&= \log \frac{p(\mathcal{B}_i)}{\guide(\mathcal{B}_i)} + \log \frac{p( \reparamXform(\reparamVar_i) | \mathcal{B}_i)}{\guide(\reparamXform(\reparamVar_i) | \mathcal{B}_i)}	+ \log \frac{p(\mathcal{D}_i | \reparamVar_i, \mathcal{B}_i)}{\guide(\mathcal{D}_i | \reparamVar_i, \mathcal{B}_i)}\\
&= w_i(\mathcal{B}_i) + w_i(\reparamVar_i, \mathcal{B}_i) + w_i(\mathcal{D}_i, \reparamVar_i, \mathcal{B}_i)
\end{align*}
%%%
In the remainder of this section, we prove that the $w_i(\mathcal{B}_i)$ term can be safely removed. Specifically, we prove the following:
%%%
\begin{theorem}
For a discrete random choice $\reparamVar_i$:
\begin{equation*}
\expect_\reparamDist [ \gradparams \log \guide(\reparamXform_i(\reparamVar_i) | \reparamXform(\reparamVars_{<i}), \observedVars) w_i(\mathcal{B}_i) ] = 0
\end{equation*}
\label{thm:wterms}
\end{theorem}
%%%
In proving this theorem, the following two lemmas will be useful:
%%%
\begin{lemma}
\begin{equation*}
\reparamVars_{<i} \cap \mathcal{D}_i = \varnothing
\end{equation*}
\label{lem:depsDisjoint}
\end{lemma}
\begin{proof}
By definition, $\reparamVars_{<i}$ are all random choices which occur before $\reparamVar_i$ in program execution order. Also by definition, for all $d \in \mathcal{D}_i$, there exists a path $\reparamVar_i \to d$. As noted above, by construction, this implies that $d$ was created after $\reparamVar_i$ in program execution order. Thus $d$ cannot be in $\reparamVars_{<i}$.
\end{proof}
%%%
\begin{lemma}
\begin{equation*}
\guide(\reparamXform_i(\reparamVar_i) | \mathcal{B}_i) \equiv \guide(\reparamXform_i(\reparamVar_i) | \mathcal{U}_i \cup \mathcal{C}_i) = \guide(\reparamXform_i(\reparamVar_i) | \mathcal{U}_i)
\end{equation*}
\label{lem:bIndependentFromU}
\end{lemma}
\begin{proof}
By construction there is no directed path $\reparamVar_i \to \mathcal{C}_i$, or vice versa.  Further, all nodes that are upstream of both $\reparamVar_i$ and $\mathcal{C}_i$ are in $\mathcal{U}_i$ (because all nodes that are upstream of $\reparamVar_i$ are in $\mathcal{U}_i$ by definition). It follows that these sets are conditionally independent:
\begin{equation*}
\guide(\reparamXform_i(\reparamVar_i), \mathcal{C}_i | \mathcal{U}_i) = \guide(\reparamXform_i(\reparamVar_i)| \mathcal{U}_i) \guide(\mathcal{C}_i | \mathcal{U}_i).
\end{equation*}
Using this, the result is immediate:
\begin{equation*}
\guide(\reparamXform_i(\reparamVar_i) | \mathcal{U}_i \cup \mathcal{C}_i) 
= \frac{\guide(\reparamXform_i(\reparamVar_i), \mathcal{C}_i | \mathcal{U}_i)}{\guide(\mathcal{C}_i | \mathcal{U}_i)} 
= \guide(\reparamXform_i(\reparamVar_i)| \mathcal{U}_i).
\end{equation*}
%\remark{Noah proves this. Is this pretty much just d-separation?}
\end{proof}
A corollary of this lemma is that $\guide(\reparamXform_i(\reparamVar_i) | \reparamXform(\reparamVars_{<i}), \observedVars) = \guide(\reparamXform_i(\reparamVar_i) | \mathcal{U}_i)$, since $\reparamVars_{<i}$ must be a subset of $\mathcal{B}_i$.
%%%
We can now prove Theorem~\ref{thm:wterms}:
%%%
\begin{align}
\expect_\reparamDist [ \gradparams \log \guide(\reparamXform_i(\reparamVar_i) | \reparamXform(\reparamVars_{<i}), \observedVars) w_i(\mathcal{B}_i) ]
&= \int_\reparamVars \reparamDist(\reparamVars | \observedVars) \gradparams \log \guide(\reparamXform_i(\reparamVar_i) | \reparamXform(\reparamVars_{<i}), \observedVars) w_i(\mathcal{B}_i)\nonumber\\
&= \int_{\mathcal{B}_i} \reparamDist(\mathcal{B}_i) w_i(\mathcal{B}_i) \sum_{\reparamVar_i} \reparamDist(\reparamVar_i | \mathcal{B}_i) \int_{\mathcal{D}_i} \reparamDist(\mathcal{D}_i | \reparamVar_i, \mathcal{B}_i) \gradparams \log \guide(\reparamXform_i(\reparamVar_i) | \reparamXform(\reparamVars_{<i}), \observedVars)\nonumber\\
\label{eq:wterms_depsDisjoint}
&= \int_{\mathcal{B}_i} \reparamDist(\mathcal{B}_i) w_i(\mathcal{B}_i) \sum_{\reparamVar_i} \reparamDist(\reparamVar_i | \mathcal{B}_i) \gradparams \log \guide(\reparamXform_i(\reparamVar_i) | \reparamXform(\reparamVars_{<i}), \observedVars) \int_{\mathcal{D}_i} \reparamDist(\mathcal{D}_i | \reparamVar_i, \mathcal{B}_i)\\
\label{eq:wterms_isDiscrete}
&= \int_{\mathcal{B}_i} \reparamDist(\mathcal{B}_i) w_i(\mathcal{B}_i) \sum_{\reparamVar_i} \guide(\reparamXform_i(\reparamVar_i) | \mathcal{B}_i) \gradparams\log \guide(\reparamXform_i(\reparamVar_i) | \reparamXform(\reparamVars_{<i}), \observedVars) \cdot 1\\
\label{eq:wterms_bIndependentFromU}
&= \int_{\mathcal{B}_i} \reparamDist(\mathcal{B}_i) w_i(\mathcal{B}_i) \sum_{\reparamVar_i} \guide(\reparamXform_i(\reparamVar_i) | \mathcal{U}_i) \gradparams \log \guide(\reparamXform_i(\reparamVar_i) | \mathcal{U}_i)\\
&= \int_{\mathcal{B}_i} \reparamDist(\mathcal{B}_i) w_i(\mathcal{B}_i) \expect_{\guide} [ \gradparams \log \guide(\reparamXform_i(\reparamVar_i) | \mathcal{U}_i) ]\nonumber\\
\label{eq:wterms_zeroexp}
&= \int_{\mathcal{B}_i} \reparamDist(\mathcal{B}_i) w_i(\mathcal{B}_i) \cdot 0 = 0
\end{align}
%%%
Line~\ref{eq:wterms_depsDisjoint} uses Lemma~\ref{lem:depsDisjoint}.
Line~\ref{eq:wterms_isDiscrete} uses the fact that $\reparamVar_i$ is discrete.
Line~\ref{eq:wterms_bIndependentFromU} uses Lemma~\ref{lem:bIndependentFromU} and its corollary.
Finally, line~\ref{eq:wterms_zeroexp} uses Lemma~\ref{lem:zeroexp}. 

We should note that similar methods to our removal of terms from $W$ have been used by prior work to reduce variance in LR gradient estimators. Rao-blackwellization, using the Markov blanket of a node in a graphical model, produces a similar estimator~\cite{BBVI}. In deep generative models, a similar technique allows the derivation of lower-variance estimators for each layer of the deep generative network~\cite{NVIL}. Finally, less conservative (i.e. exact) dependency graphs can be used to reduce gradient variance in stochastic computation graphs~\cite{StochasticComputationGraphs}.

\subsection{Variance Reduction Step 2: Baselines}

Next, we prove that subtracting a constant baseline term $b_i$ from every $w_i$ does not change the expectation in Equation~\ref{eq:finalEstimator}:
%%%
\begin{align*}
\expect_\reparamDist [ \gradparams \log \guide(\reparamXform_i(\reparamVar_i) | \reparamXform(\reparamVars_{<i})) (w_i(\reparamVars, \observedVars) - b_i) ]
&= \expect_\reparamDist [ \gradparams \log \guide(\reparamXform_i(\reparamVar_i) | \reparamXform(\reparamVars_{<i})) w_i(\reparamVars, \observedVars) ] - \expect_\reparamDist[ \gradparams \log \guide(\reparamXform_i(\reparamVar_i) | \reparamXform(\reparamVars_{<i})) b_i ]\\
&= \expect_\reparamDist [ \gradparams \log \guide(\reparamXform_i(\reparamVar_i) | \reparamXform(\reparamVars_{<i})) w_i(\reparamVars, \observedVars) ] 
\end{align*}
%%%
Where the last step makes use of Lemma~\ref{lem:zeroexp2}.

In our system, we use $b_i = \expect[ w_i ]$, which we estimate with of a moving average of the samples used to compute gradients. While this choice of $b_i$ is not guaranteed to reduce variance, it often does in practice, and previous systems for variational inference and reinforcement learning have exploited it~\cite{BBVI,StochasticComputationGraphs,VarianceReduction}. Another option is to \emph{learn} $b_i$, for example as a neural net function of $\observedVars$~\cite{NVIL}. The proof above also permits $b_i$ to be a function of $\reparamVars_{<i}$ (i.e. all previous random choices), which could reduce variance further by tracking posterior dependencies. This is a promising avenue for future work.

\subsection{Variance Reduction Step 3: Zero Expectation $q$ Factors}

Finally, we prove that we can remove any factors corresponding to discrete (i.e. non-reparameterized choices) from the $\gradparams \log \guide(\xformedVars | \observedVars)$ term in Equation~\ref{eq:hybridEstimator} without changing its expectation:
%%%
\begin{equation*}
\expect_\reparamDist [ \gradparams \log \guide(\reparamXform_i(\reparamVar_i) | \reparamXform_{<i}(\reparamVars_{<i}), \observedVars) ]
= \expect_\reparamDist [ \gradparams \log \reparamDist(\reparamVar_i | \reparamVars_{<i}, \observedVars) ]
= 0
\end{equation*}
where we have used Lemma~\ref{lem:zeroexp2} and the fact that $\reparamDist = \guide$ and $\reparamXform$ is the identity for discrete random choices.
%%%

%auto-ignore

\section{Appendix: Example Programs}
\label{sec:appendix_code}

\subsection{Gaussian mixture model}
\label{sec:appendix_code:gmmSumOut}

With discrete choices marginalized out:
\begin{lstlisting}
var obs = loadData('data.json');
var nComps = 3
var model = function() {
   var theta_x = simplex(modelParam({dims: [nComps-1, 1], name: 'theta_x'}));
   var params_y = [
      {mu: modelParam({name: 'mu1'}), sigma: softplus(modelParam({name: 's1'}))},
      {mu: modelParam({name: 'mu2'}), sigma: softplus(modelParam({name: 's2'}))},
      {mu: modelParam({name: 'mu3'}), sigma: softplus(modelParam({name: 's3'}))}
   ];
   mapData({data: obs}, function(y) {
      // Explicitly sum out latent mixture component
      var scores = mapIndexed(function(i, muSigma) {
         var w = T.get(theta_x, i);
         return Gaussian(muSigma).score(y) + Math.log(w);
      }, params_y);
      // Summed-out likelihod
      factor(logsumexp(scores));
   });
};
\end{lstlisting}

\subsection{QMR-DT}
\label{sec:appendix_code:qmr}

Program with joint guide:
%%%
\begin{lstlisting}
var graph = loadQMRGraph('qmr_graph.json');
var data = loadData('qmr_data.json');

var noisyOrProb = function(symptomNode, diseases) {
   var cp = product(map(function(parent) {
      return diseases[parent.index] ? (1 - parent.prob) : 1;
   }, symptomNode.parents));
   return 1 - (1-symptomNode.leakProb)*cp;
};

var guideNet = nn.mlp(graph.numSymptoms, [
   {nOut: graph.numDiseases, activation: sigmoid}
], 'guideNet');
var predictDiseaseProbs = function(symptoms) {
   return nneval(guideNet, Vector(symptoms));
};

var model = function() {
   mapData({data: data, batchSize: 20}, function(symptoms) {
      var predictedProbs = predictDiseaseProbs(symptoms);
      var diseases = mapIndexed(function(i, disease) {
         return sample(Bernoulli({p: disease.priorProb}), {
            guide: Bernoulli({p: T.get(predictedProbs, i)})
         });
      }, graph.diseaseNodes);

      mapData({data: symptoms}, function(symptom, symptomIndex) {
         var symptomNode = graph.symptomNodes[symptomIndex];
         observe(Bernoulli({p: noisyOrProb(symptomNode, diseases)}), symptom);
      });
   });
};
\end{lstlisting}

Program with factored guide. Note that the guide uses a separate neural network (with separate parameters) to predict each latent cause.
%%%
\begin{lstlisting}
var graph = loadQMRGraph('qmr_graph.json');
var data = loadData('qmr_data.json');

var noisyOrProb = function(symptomNode, diseases) {
   var cp = product(map(function(parent) {
      return diseases[parent.index] ? (1 - parent.prob) : 1;
   }, symptomNode.parents));
   return 1 - (1-symptomNode.leakProb)*cp;
};

var predictNet = cache(function(i) {
   return nn.mlp(graph.numSymptoms, [
      {nOut: 1, activation: sigmoid}
   ], 'predictNet_'+i);
});
var predictDiseaseProb = function(symptoms, i) {
   return T.get(nneval(predictNet(i), Vector(symptoms)), 0);
};

var model = function() {
   mapData({data: data, batchSize: 20}, function(symptoms) {
      var diseases = mapIndexed(function(i, disease) {
         return sample(Bernoulli({p: disease.priorProb}), {
            guide: Bernoulli({p: predictDiseaseProb(symptoms, i)})
         });
      }, graph.diseaseNodes);

      mapData({data: symptoms}, function(symptom, symptomIndex) {
         var symptomNode = graph.symptomNodes[symptomIndex];
         observe(Bernoulli({p: noisyOrProb(symptomNode, diseases)}), symptom);
      });
   });
};
\end{lstlisting}

Program with factored guide with GRU.
%%%
\begin{lstlisting}
var graph = loadQMRGraph('qmr_graph.json');
var data = loadData('qmr_data.json');
var gruHiddenDim = 20;

var noisyOrProb = function(symptomNode, diseases) {
   var cp = product(map(function(parent) {
      return diseases[parent.index] ? (1 - parent.prob) : 1;
   }, symptomNode.parents));
   return 1 - (1-symptomNode.leakProb)*cp;
};

var predictNet = cache(function(i) {
   return nn.mlp(graph.numSymptoms + gruHiddenDim, [
      {nOut: 1, activation: sigmoid}
   ], 'predictNet_'+i);
});
var predictDiseaseProb = function(symptoms, i) {
   var inputs = T.concat(Vector(symptoms), globalStore.gruHiddenState);
   return T.get(nneval(predictNet(i), inputs, 0);
};

// GRU cell: takes input + previous hidden state, produces new hidden state.
var gru = makeGRU(gruHiddenDim, 1, 'gru');
var updateGRUState = function(latentChoiceVal) {
   var inputs = T.concat(Vector([latentChoiceVal]), globalStore.gruHiddenState);
   globalStore.gruHiddenState = nneval(gru, inputs);
};

var model = function() {
   mapData({data: data, batchSize: 20}, function(symptoms) {
      globalStore.gruHiddenState = zeros([gruHiddenDim]);
      var diseases = mapIndexed(function(i, disease) {
         var x = sample(Bernoulli({p: disease.priorProb}), {
            guide: Bernoulli({p: predictDiseaseProb(symptoms, i)})
         });
         updateGRUState(x);
         return x;
      }, graph.diseaseNodes);

      mapData({data: symptoms}, function(symptom, symptomIndex) {
         var symptomNode = graph.symptomNodes[symptomIndex];
         observe(Bernoulli({p: noisyOrProb(symptomNode, diseases)}), symptom);
      });
   });
};
\end{lstlisting}

\subsection{LDA}
\label{sec:appendix_code:lda}

To simplify the programs in this section, we use the technique of Section~\ref{sec:autoGuide:meanField} to insert a mean field \ic{LogisticNormal} guide for any \ic{Dirichlet} random choice that does not have an explicit guide declared.

Mean field model:
\begin{lstlisting}
var model = function(corpus, vocabSize, numTopics, alpha, eta) {
  var topics = repeat(numTopics, function() {
    return sample(Dirichlet({alpha: eta}));
  });

  mapData({data: corpus}, function(doc) {

    var topicDist = sample(Dirichlet({alpha: alpha}));

    mapData({data: doc}, function(word) {
      var z = sample(Discrete({ps: topicDist}));
      var topic = topics[z];
      observe(Discrete({ps: topic}), word);
    });

  });

  return topics;
};
\end{lstlisting}

Marginalized mean field model:
\begin{lstlisting}
var model = function(corpus, vocabSize, numTopics, alpha, eta) {

  var topics = repeat(numTopics, function() {
    return sample(Dirichlet({alpha: eta}));
  });

  mapData({data: corpus}, function(doc) {

    var topicDist = sample(Dirichlet({alpha: alpha}));

    forEach(doc, function(count, word) {

      if (count > 0) {
        // Sum over topic assignments/z.
        var prob = sum(mapN(function(z) {
          var zScore = Discrete({ps: topicDist}).score(z);
          var wgivenzScore = Discrete({ps: topics[z]}).score(word);
          return Math.exp(zScore + wgivenzScore);
        }, numTopics));

        factor(Math.log(prob) * count);
      }

    });

  });

  return topics;

};
\end{lstlisting}

Word-level guide:
\begin{lstlisting}
var model = function(corpus, vocabSize, numTopics, alpha, eta) {

  var numHid = 50;
  var embedSize = 50;
  var embedNet = nn.mlp(vocabSize, [{nOut: embedSize, activation: nn.tanh}], 'embedNet');

  var net = nn.mlp(embedSize + numTopics, [
    {nOut: numHid, activation: nn.tanh},
    {nOut: numTopics}
  ], 'net');

  var wordAndTopicDistToParams = function(word, topicDist) {
    var embedding = nneval(embedNet, oneOfK(word, vocabSize));
    var out = nneval(net, T.concat(embedding, T.sub(topicDist, 1)));
    return {ps: softplus(tensorToVector(out))};
  };

  var topics = repeat(numTopics, function() {
    return sample(Dirichlet({alpha: eta}));
  });

  mapData({data: corpus}, function(doc) {

    var topicDist = sample(Dirichlet({alpha: alpha}));

    mapData({data: doc}, function(word) {
      var z = sample(Discrete({ps: topicDist}), {
        guide: Discrete(wordAndTopicDistToParams(word, topicDist))
      });
      var topic = topics[z];
      observe(Discrete({ps: topic}), word);
    });

  });

  return topics;
};
\end{lstlisting}

Document-level guide:
\begin{lstlisting}
var nets = cache(function(numHid, vocabSize, numTopics) {
  var init = nn.constantparams([numHid], 'init');

  var ru = makeRNN(numHid, vocabSize, 'ru');

  var outputHidden = nn.mlp(numHid, [
    {nOut: numHid, activation: nn.tanh}
  ], 'outputHidden');

  var outputMu = nn.mlp(numHid, [
    {nOut: numTopics - 1}
  ], 'outputMu');

  var outputSigma = nn.mlp(numHid, [
    {nOut: numTopics - 1}
  ], 'outputSigma');

  return {
    init: init,
    ru: ru,
    outputHidden: outputHidden,
    outputMu: outputMu,
    outputSigma: outputSigma
  };
});

var model = function(data, vocabSize, numTopics, alpha, eta) {
  var corpus = data.documentsAsCounts;
  var numHid = 20;
  var nets = nets(numHid, vocabSize, numTopics);

  var guideParams = function(topics, doc) {
    var initialState = nneval(nets.init);
    var state = reduce(function(x, prevState) {
      return nneval(nets.ru, [prevState, x]);
    }, initialState, topics.concat(normalize(Vector(doc))));
    var hidden = nneval(nets.outputHidden, state);
    var mu = tensorToVector(nneval(nets.outputMu, hidden));
    var sigma = tensorToVector(softplus(nneval(nets.outputSigma, hidden)));
    var params = {mu: mu, sigma: sigma};
    return params;
  };

  var topics = repeat(numTopics, function() {
    return sample(Dirichlet({alpha: eta}));
  });

  mapData({data: corpus}, function(doc) {

    var topicDist = sample(Dirichlet({alpha: alpha}), {
      guide: LogisticNormal(guideParams(topics, doc))
    });

    mapData({data: countsToIndices(doc)}, function(word) {
      var z = sample(Discrete({ps: topicDist}));
      var topic = topics[z];
      observe(Discrete({ps: topic}), word);
    });

  });

  return topics;
};
\end{lstlisting}

\subsection{Variational Autoencoder}
\label{sec:appendix_code:vae}

In this example and the one that follows, \ic{nnevalModel} evaluates a neural network while also placing an improper uniform prior over the network parameters. This allows neural networks to be used as part of learnable models.

\begin{lstlisting}
// Data
var data = loadData('mnist.json');
var dataDim = 28*28;
var hiddenDim = 500;
var latentDim = 20;

// Encoder
var encodeNet = nn.mlp(dataDim, [
   {nOut: hiddenDim, activation: nn.tanh}
], 'encodeNet');
var muNet = nn.linear(hiddenDim, latentDim, 'muNet');
var sigmaNet = nn.linear(hiddenDim, latentDim, 'sigmaNet');
var encode = function(image) {
   var h = nneval(encodeNet, image);
   return {
      mu: nneval(muNet, h),
      sigma: softplus(nneval(sigmaNet, h))
   };
};

// Decoder
var decodeNet = nn.mlp(latentDim, [
   {nOut: hiddenDim, activation: nn.tanh},
   {nOut: dataDim, activation: nn.sigmoid}
], 'decodeNet');
var decode = function(latent) {
   return nnevalModel(decodeNet, latent);
};

// Training model
var model = function() {
   mapData({data: data, batchSize: 100}, function(image) {
      // Sample latent code (guided by encoder)
      var latent = sample(TensorGaussian({mu: 0, sigma: 1, dims: [latentDim, 1]}), {
         guide: DiagCovGaussian(encode(image))
      });

      // Decode latent code, observe binary image
      var probs = decode(latent);
      observe(MultivariateBernoulli({ps: probs}), image);
   });
}
\end{lstlisting}

\subsection{Sigmoid Belief Network}
\label{sec:appendix_code:sbn}

\begin{lstlisting}
// Data
var data = loadData('mnist.json');
var dataDim = 28*28;
var latentDim = 200;

// Encoder
var encodeNet = nn.mlp(dataDim, [
   {nOut: latentDim, activation: nn.sigmoid}
], 'encodeNet');
var encode = function(image) {
   return nneval(encodeNet, image)
};

// Decoder
var decodeNet = nn.mlp(latentDim, [
   {nOut: dataDim, activation: nn.sigmoid}
], 'decodeNet');
var decode = function(latent) {
   return nnevalModel(decodeNet, latent);
};

// Training model
var priorProbs = Vector(repeat(latentDim, function() { return 0.5; }));
var model = function() {
   mapData({data: data, batchSize: 100}, function(image) {
      // Sample latent code (guided by encoder)
      var latent = sample(MultivariateBernoulli({ps: priorProbs}), {
         guide: MultivariateBernoulli({ps: encode(image)})
      });

      // Decode latent code, observe binary image
      var probs = decode(latent);
      observe(MultivariateBernoulli({ps: probs}), image);
   });
}
\end{lstlisting}

\end{document}